\documentclass{article}

\newif\ifDoubleColumn

\ifDoubleColumn
	\usepackage{rx_macros}
\else 
    \usepackage{rx_macros}
\fi

\allowdisplaybreaks

\ifDoubleColumn
    \relax
\else
    \title{\bf Cascaded Gaps: Towards Gap-Dependent Regret for Risk-Sensitive Reinforcement Learning}
    \author[1]{Yingjie Fei\thanks{The authors are presented in alphabetical order.}}
    \affil[1]{Bloomberg}
    \author[2]{Ruitu Xu$^*$}
    \affil[2]{Department of Statistics and Data Science, 
    Yale University}
    \date{}
\fi

\begin{document}

\ifDoubleColumn
    \relax
\else
    \footnotetext[1]{Email: \texttt{yf275@cornell.edu}}
    \footnotetext[2]{Email: \texttt{ruitu.xu@yale.edu}}
\fi

\ifDoubleColumn
    \maketitle
\else
    \maketitle
\fi

\begin{abstract}
    In this paper, we study gap-dependent regret guarantees for risk-sensitive reinforcement learning based on the entropic risk measure. We propose a novel definition of sub-optimality gaps, which we call  cascaded gaps, and we discuss their key components that adapt to the underlying structures of the problem. Based on the cascaded gaps, we derive non-asymptotic and logarithmic regret bounds for two model-free algorithms under episodic Markov decision processes. We show that, in appropriate settings, these bounds feature exponential improvement over existing ones that are independent of gaps. We also prove gap-dependent lower bounds, which certify the near optimality of the upper bounds. 
\end{abstract}

\section{Introduction}

We study the problem of risk-sensitive \gls*{RL} based on the entropic risk measure, in which we aim to identify a decision making rule (or policy) $\widehat{\pi}$ that solves the following optimization problem:
\begin{align}
    \max_{\pi} 
    \Big\{
    V^\pi = 
    \frac{1}{\beta} \log(\expect_\pi e^{\beta R})
    \Big\},
    \label{eq:risk_sens_informal}
\end{align}
where $R$ denotes the cumulative reward  and $\beta\neq 0$ is the risk parameter that induces risk-seeking learning when $\beta>0$ and risk-averse learning when $\beta<0$. The (standard) risk-neutral objective function used in RL, which is simply $\E_{\pi}[R]$, can be recovered from  \cref{eq:risk_sens_informal} by setting $\beta \to 0$. Moreover, the objective of \eqref{eq:risk_sens_informal} in the form of entropic risk measure admits a Taylor expansion $V^{\pi} = \E_\pi[R] + \frac{\beta}{2} \text{Var}_{\pi}(R) + O(\beta^2)$, which represents a trade-off between the expectation and the variance (and possibly higher-order statistics) of the reward.
Several lines of research on related problems have witnessed fruitful applications in a wide range of domains, including neuroscience  \citep{niv2012neural,shen2014risk},  robotics \citep{nass2019entropic,williams2016aggressive, williams2017information}, economics \citep{hansen2011robustness},  and etc. The formulation \eqref{eq:risk_sens_informal} has been related to notions of robustness \citep{osogami2012robustness,hansen2011robustness,follmer2011entropic} and bounded rationality \citep{simon1955behavioral,ortega2016human} in decision making and behavioral studies. A thermodynamic view on such formulation has also been proposed for understanding  sequential decision making systems \citep{ortega2013thermodynamics}.

For problem \eqref{eq:risk_sens_informal}, much recent work has been devoted to designing algorithms that attain finite-sample regret bounds under \glspl*{MDP}. Although the existing bounds are nearly optimal in the minimax sense,
they are overly pessimistic as they generally  fail to exploit particular structures of the underlying \glspl*{MDP}, such as sub-optimality gaps, which quantify the easiness of learning optimal policies under the \glspl*{MDP}.
Although previous work has explored and provided gap-dependent results for risk-neutral \gls*{RL}, 
it is unclear how the sub-optimality gaps should be constructed in the risk-sensitive setting. 
In particular, 
the definition of existing sub-optimality gaps, as we will elaborate in \cref{sec:cascaded-gaps}, crucially hinges on the linear structures of the risk-neutral setting, which no longer hold in the risk-sensitive setting characterized by the non-linear objective  \eqref{eq:risk_sens_informal}.
It therefore begs the following natural questions: 
1) how sub-optimality  gaps should be characterized in risk-sensitive \gls*{RL}, and 
2) whether we can obtain refined bounds on regret and sample complexity by taking advantage of the gap structures.

To answer the above questions, we study gap-dependent regret bounds for risk-sensitive \gls*{RL} based on the entropic risk measure. In particular, we identify two key conditions for a proper definition of sub-optimality gaps for risk-sensitive \gls*{RL}: Bellman difference condition and risk consistency condition. The Bellman difference condition states that the gaps induce a Bellman equation in which they play the role of reward functions; the risk consistency condition stipulates that the gaps stay on the same order of magnitude for both risk-averse and risk-seeking settings given fixed risk sensitivity $|\beta|$, and they reduce to risk-neutral gaps as $|\beta|$ vanishes.
Motivated by the two conditions, we propose a novel characterization of sub-optimality gaps for risk-sensitive \gls*{RL}, which we call \emph{cascaded gaps}. Cascaded gaps consist of three key components: 1) the difference of rewards along trajectories controlled by an optimal policy, 2)  the reward functions evaluated along a free trajectory (not controlled by any policy), and 3)  a normalization factor that depends on the risk parameter. 
The first two components together exhibit a cascading property and address the Bellman difference condition, whereas the third component facilitates risk consistency.

Based on the cascaded gaps, 
we derive  non-asymptotic regret bounds for two existing risk-sensitive \gls*{RL} algorithms, \textsc{RSVI2} and \textsc{RSQ2}, that scale logarithmically in the number of episodes and decay in the cascaded gaps. The proof is based on a unified framework for both algorithms. We demonstrate that under proper settings, our regret bounds attain an exponential improvement over existing results with respect to the number of episodes, as well as an exponential improvement in terms of risk sensitivity and horizon over existing sample complexity bounds.
We further show that the provided upper bounds are nearly optimal by deriving compatible lower bounds.
To the best of our knowledge, this is the first work that studies  sub-optimality gaps in risk-sensitive \gls*{RL} with the  entropic risk measure and derives  gap-dependent regret bounds.

\paragraph{Contributions.}
In summary, we make the following theoretical contributions in this paper:
\begin{enumerate}
    \item We propose a novel notion of sub-optimality gaps for risk-sensitive \gls*{RL} based on the entropic risk measure, which we call 
    cascaded gaps. We discuss  essential components of  cascaded gaps tailored to the unique structure 
    of risk-sensitive RL, and compare them with  sub-optimality gaps in the risk-neutral setting.
    
    \item We prove logarithmic regret bounds that adapt to the sub-optimality gaps for two existing risk-sensitive \gls*{RL} algorithms. The bounds are achieved via a unified framework for both algorithms,  
    and they imply exponential improvements in both regret and sample complexity under appropriate settings.
    \item We further derive lower bounds that nearly match the upper bounds, thereby showing that the upper bounds are nearly optimal.

\end{enumerate}

\paragraph{Notation.}
We write shorthand $[n] \defeq 1,\ldots,n$ for any $n\in\ZZ_{+}$. 
For any series of variables $\{v_i\}_{i \in [n]}$, 
we define the notation $\poly(v_1,\ldots,v_n) \defeq c_0 \prod_{i\in[n]}v_i^{c_i}$ and  $\polylog(v_1,\ldots,v_n) \defeq c_0\prod_{i\in[n]}\log(v_i)^{c_i}$ for some positive universal constants $\{c_i\}_{i \ge 0}$.
For $x>0$, we write $\widetilde O(x)$ to denote $O(x \polylog(x))$; we define $\widetilde\Omega(x)$ in a similar way.   Unless otherwise specified, $\log$ denotes the natural logarithm and $\log_2$ denotes the logarithm with base $2$. For any functions $f$ and $g$ with the same domain, we write $f\leq g$ to mean $f(y)\leq g(y)$ for all $y$ in the domain. We use notation $\phi(n)\lesssim \varphi(n)$ (or $\phi(n)\gtrsim \varphi(n)$) for functions $\phi$ and $\varphi$ that satisfy $\phi(n)\leq C\varphi(n)$ (or $\phi(n)\geq c \varphi(n)$) for every $n\in\ZZ_+$ with some universal constant $C>0$ (or $c>0$); further, we write $\phi(n)\asymp \varphi(n)$ to mean $\phi(n)\lesssim \varphi(n)$ and $\phi(n)\gtrsim \varphi(n)$.

\section{Related Works}
Initiated by \citet{howard1972risk,jacobson1973optimal}, risk-sensitive RL based on the entropic risk measure has been the focus of long-standing research efforts for the past decades \citep{borkar2002q,borkar2002risk,coraluppi1999risk,osogami2012robustness, shen2013risk, whittle1990risk,mihatsch2002risk, bauerle2014more,fleming1995risk, di1999risk, borkar2001sensitivity}. Most related to our work are perhaps those by \citet{fei2020risk,fei2021risk,fei2021exponential}: under the episodic and finite-horizon \glspl*{MDP}, they propose computationally efficient algorithms for risk-sensitive RL and provide finite-sample and nearly optimal regret guarantees in both tabular and linear settings. These results are general, holding without access to transitions or simulators. However, they fail to exploit particular structures of the underlying \glspl*{MDP}, such as sub-optimality gaps, and are therefore overly conservative under certain settings.

For risk-neutral \gls*{RL}, a series of works has established non-asymptotic and  gap-dependent regret bounds for optimistic algorithms, starting from \citet{simchowitz2019non}. Specifically, logarithmic regret are derived for optimistic Q-learning \citep{yang2021q} and value iteration \citep{he2021logarithmic}.
Despite these recent developments, it remains unclear whether the defintion of sub-optimality gaps in  risk-neutral RL is appropriate for the risk-sensitive setting, which the present work aims to address.

\section{Preliminaries}

\subsection{Episodic and Finite-Horizon MDPs}

We focus on the setting of tabular \glspl*{MDP}, represented by a tuple $(\calS,\calA,H,K,\calP,r)$. Here,  $\calS$ denotes the set of available states with cardinality $|\calS| = S$, $\calA$ the set of actions available to the agent with cardinality $|\calA| = A$, $K$ the number of episodes, $H$  the horizon,
$\calP = \{\calP_h\}_{h \in [H]}$ the set of transition kernels, and $r = \{r_h\}_{h \in [H]}$ the set of reward functions. We assume that reward $r_h:\calS\times\calA\to[0,1]$ is deterministic for every step $h$.
Without loss of generality, the agent starts at a fixed state $s_1^k = s_1$ in each episode $k \in [K]$. For  episode $k\in[K]$ and step $h \in [H]$, it takes action $a^k_h$ at state $s^k_h$ and receives reward $r^k_h(s^k_h,a^k_h)$. Then the  environment transitions into $s^k_{h+1}$ with probability equal to $\calP_h(s^k_{h+1}| s^k_h,a^k_h)$.

\subsection{Risk-Sensitive RL}

We define policy $\pi \defeq \{\pi_h: \cS\to\cA\}$ as a collection of functions that map states to actions. 
In risk-sensitive \gls*{RL} based on entropic risk measure, we define the state-value function with respect to any $\pi$:
\begin{align*}
    \Vpih(s) \defeq \frac{1}{\beta} \log\left\{ \expect\Big[e^{\beta \sum_{i=h}^H r_{i}(s_{i},\pi_{i}(s_{i}))}\Big] \given s_h=s \right\},
\end{align*}
for each $h\in[H]$ and $s\in\calS$, where the expectation is taken over the transition kernel $\calP$. The quantity $\beta\neq 0$ is the risk parameter of the entropic risk measure. In particular, $\beta>0$ yields a risk-seeking value function, while $\beta<0$ induces a risk-averse value function. 
The risk-neutral definition of the value function $\widetilde{V}^{\pi}_{h}(s) \defeq \expect[\sum_{i=h}^H r_{i}(s_{i},\pi_{i}(s_{i}))\given s_h=s]$ can be recovered through taking $\beta\to 0$.
Similarly, we define the corresponding action-value function as
\ifDoubleColumn
    \begin{align*}
        & \Qpih(s,a)  \\
        & \quad \defeq \frac{1}{\beta} \log\Big\{  \expect\Big[e^{\beta \sum_{i=h}^H r_{i}(s_{i},\pi_{i}(s_{i}))}\Big] \given s_h=s, a_h=a \Big\}.
    \end{align*}
\else
    \begin{align*}
        \Qpih(s,a) & \defeq \frac{1}{\beta} \log\Big\{  \expect\Big[e^{\beta \sum_{i=h}^H r_{i}(s_{i},\pi_{i}(s_{i}))}\Big] \given s_h=s, a_h=a \Big\}.
    \end{align*}
\fi
Note that we omit the dependency of $V_{h}^{\pi}$ and $Q_{h}^{\pi}$ on $\beta$ for simplicity. 
Consequently, the Bellman equation for risk-sensitive RL is given by
\begin{align}\label{eqn:bellman}
    \Qpih(s,a) = r_h(s,a) + \frac{1}{\beta}\log
    \expectph\big[ e^{\beta\cdot\Vpihh(s')} \big],
\end{align}
which relates the action-value function $\Qpih$ to the state-value function $\Vpihh$ of the next step. Note that the Bellman equation is non-linear in the value function due to the non-linearity of the entropic risk measure.
It can be shown that there always exists an optimal policy $\pi^*$ with the optimal value  $\Vsh(s) \defeq V^{\pi^*}_h(s) = \sup_\pi\Vpih(s)$ for every $h\in[H]$ and $s\in\calS$; we also write $Q^{*}_{h} \defeq Q^{\pi^*}_h$ for $h \in [H]$.

Under episodic \glspl*{MDP}, the agent aims to learn an optimal policy $\pi^*$ by interacting with the environment for $K$ episodes. We measure the performance of the agent that follows  policies $\{\pi^k\}_{k\in[K]}$ via the notion of  regret, which is defined as 
\begin{align*}
        \calR(K) \defeq \sumk (V_1^* - V_1^{\pi^k})(s_1^k).
\end{align*}

\section{Cascaded Gaps}\label{sec:cascaded-gaps}

\subsection{Bellman Difference Condition}\label{sec:bellman-diff-cond}

Since both regret and sub-optimality gaps represent some notion of sub-optimality with respect to an optimal policy $\pi^*$, it would be instrumental to  associate the two through a unified lens. We do so by introducing the following condition, which later plays a key role in our analysis.

\begin{cond}[Bellman Difference Condition] \label{cond:bellman_diff}
We say that gap functions $\{\gp_{h}:\cS\times\cA\to\real\}_{h\in[H]}$ 
satisfy the \textit{Bellman difference condition} 
if, for any policy $\pi$ and tuple $(h,s)\in [H]\times \cS$, there exists
some $Z_{h}^{\pi}:\cS \to \real$
such that
\begin{align*}
     \vdiff_{h}^{\pi}(s) & = \gp_{h}(s,a) + \E_{s'\sim \calP_h(\cdot|s,a)}[\vdiff_{h+1}^{\pi}(s')],
\end{align*}
where $\vdiff_{h}^{\pi}  \defeq  Z^{*}_h - Z^{\pi}_h$ and $a \defeq \pi_h(s)$.
\end{cond}

\cref{cond:bellman_diff}  stipulates that for any fixed policy $\pi$, the gaps induce a form of Bellman equation where the action follows policy $\pi$. The  function $\vdiff^{\pi}_h$, itself being the difference of two functionals with respect to some  $\widehat{\pi}^*$ (optimal with respect to $\{Z^{\pi}_h\}$) and $\pi$, takes the role of the value function, and the gap takes the role of the reward function. 
Indeed, \cref{cond:bellman_diff} associates the sub-optimality induced by $\vdiff^{\pi}_h$ with that embedded in $\gp_h$. 
The condition also suggests that when $\pi = \widehat{\pi}^*$, we have $\vdiff_h^{\widehat{\pi}^*} = 0$ and therefore $\gp_h(s, \widehat{\pi}^*_h(s)) = 0$.

As an example, we show that the sub-optimality gaps defined in risk-neutral RL meets \cref{cond:bellman_diff}. Recall  the risk-neutral value functions 
\ifDoubleColumn
    \begin{align*}
        \widetilde{Q}^{\pi}_h(s,a) & \defeq \expect\Bigg[\sum_{i=h}^H r_{i}(s_{i},\pi_{i}(s_{i})) \given s_h=s, a_h=a\Bigg], \\
        \widetilde{V}^{\pi}_h(s) & \defeq \widetilde{Q}^{\pi}_h(s,\pi_h(s)),
    \end{align*}
\else
    \begin{align*}
        \widetilde{Q}^{\pi}_h(s,a) \defeq \expect\Bigg[\sum_{i=h}^H r_{i}(s_{i},\pi_{i}(s_{i})) \given s_h=s, a_h=a\Bigg], \qquad \widetilde{V}^{\pi}_h(s) \defeq \widetilde{Q}^{\pi}_h(s,\pi_h(s)),
    \end{align*}
\fi
for  $(h,s,a)\in[H]\times\cS\times\cA$ and policy $\pi$, with $\widetilde{Q}^*_h$ and $\widetilde{V}^*_h$ being the corresponding optimal value functions. In existing literature, the sub-optimality gaps for risk-neutral \gls*{RL} are given by  
\begin{align}
\widetilde{\Delta}_h(s,a) \defeq \widetilde{V}^{*}_h(s) - \widetilde{Q}^{*}_h(s,a), \label{eq:gap_risk_neut}
\end{align}
for all $(h,s,a)\in [H]\times \cS \times \cA$ \citep{simchowitz2019non,yang2021q,he2021logarithmic}. Note that the gap  in  \cref{eq:gap_risk_neut} computes the  difference in values between the optimal action $\pi^*(s)$ and action $a$. 
As stated and proved below, it satisfies \cref{cond:bellman_diff} in the risk-neutral setting.  
\begin{fact}
\label{fact:risk-neut-gap-meet-cond}
    The sub-optimality gaps $\{\widetilde{\Delta}_h\}_{h \in [H]}$ for risk-neutral \gls*{RL} satisfy \cref{cond:bellman_diff} with $Z^{\pi}_h \defeq \widetilde{V}^{\pi}_h$. 
\end{fact}
\begin{proof}
Recall that in the  risk-neutral setting, the Bellman equation for any policy $\pi$
is given by 
\begin{align}
    \widetilde{Q}_{h}^{\pi}(s,a')=r_{h}(s,a')+\E_{s'\sim \calP_h(\cdot|s,a')}[\widetilde{V}_{h+1}^{\pi}(s')]
    \label{eq:bellman_risk_neut}
\end{align}
for any $(h,s,a')\in [H]\times \cS \times \cA$.
We fix a tuple $(h, s,a)$ where $a = \pi_h(s)$, and let $Z^{\pi}_h \defeq \widetilde{V}^{\pi}_h$ so that  $\vdiff^{\pi}_{h} =  \widetilde{V}_{h}^{*}-\widetilde{V}_{h}^{\pi}$ in \cref{cond:bellman_diff}. 
From the definition \eqref{eq:gap_risk_neut} of $\widetilde{\Delta}_h$, we have
\ifDoubleColumn
    \begin{align*}
        \widetilde{\Delta}_h(s,a) &=  \widetilde{V}_{h}^{*}(s)-\widetilde{Q}_{h}^{*}(s,a)\\
         & =\widetilde{V}_{h}^{*}(s)-\widetilde{V}_{h}^{\pi}(s)+\widetilde{V}_{h}^{\pi}(s)-\widetilde{Q}_{h}^{*}(s,a)\\
         & =\widetilde{V}_{h}^{*}(s)-\widetilde{V}_{h}^{\pi}(s)+\widetilde{Q}_{h}^{\pi}(s,a)-\widetilde{Q}_{h}^{*}(s,a)\\
         & \labelrel={eqn:bellman-decomp}\widetilde{V}_{h}^{*}(s)-\widetilde{V}_{h}^{\pi}(s) \\
         & \quad +\left[r_{h}(s,a)+\E_{s'\sim \calP_h(\cdot|s,a)}[\widetilde{V}_{h+1}^{\pi}(s')]\right] \\
         & \quad -\left[r_{h}(s,a)+\E_{s'\sim \calP_h(\cdot|s,a)}[\widetilde{V}_{h+1}^{*}(s')]\right]\\
         & =\widetilde{V}_{h}^{*}(s)-\widetilde{V}_{h}^{\pi}(s)\\
         & \quad -\E_{s'\sim \calP_h(\cdot|s,a)}[\widetilde{V}_{h+1}^{*}(s')-\widetilde{V}_{h+1}^{\pi}(s')]\\
         & =\vdiff^{\pi}_{h}(s)-\E_{s'\sim \calP_h(\cdot|s,a)}[\vdiff^{\pi}_{h+1}(s')],
    \end{align*}
\else
    \begin{align*}
        \widetilde{\Delta}_h(s,a) &=  \widetilde{V}_{h}^{*}(s)-\widetilde{Q}_{h}^{*}(s,a)\\
         & =\widetilde{V}_{h}^{*}(s)-\widetilde{V}_{h}^{\pi}(s)+\widetilde{V}_{h}^{\pi}(s)-\widetilde{Q}_{h}^{*}(s,a)\\
         & =\widetilde{V}_{h}^{*}(s)-\widetilde{V}_{h}^{\pi}(s)+\widetilde{Q}_{h}^{\pi}(s,a)-\widetilde{Q}_{h}^{*}(s,a)\\
         & \labelrel={eqn:bellman-decomp}\widetilde{V}_{h}^{*}(s)-\widetilde{V}_{h}^{\pi}(s)+\left[r_{h}(s,a)+\E_{s'\sim \calP_h(\cdot|s,a)}[\widetilde{V}_{h+1}^{\pi}(s')]\right]-\left[r_{h}(s,a)+\E_{s'\sim \calP_h(\cdot|s,a)}[\widetilde{V}_{h+1}^{*}(s')]\right]\\
         & =\widetilde{V}_{h}^{*}(s)-\widetilde{V}_{h}^{\pi}(s)-\E_{s'\sim \calP_h(\cdot|s,a)}[\widetilde{V}_{h+1}^{*}(s')-\widetilde{V}_{h+1}^{\pi}(s')]\\
         & =\vdiff^{\pi}_{h}(s)-\E_{s'\sim \calP_h(\cdot|s,a)}[\vdiff^{\pi}_{h+1}(s')],
    \end{align*}
\fi
where step \eqref{eqn:bellman-decomp} is due to the Bellman equation
\eqref{eq:bellman_risk_neut}.
\end{proof}
Given \cref{fact:risk-neut-gap-meet-cond}, a connection between regret and sub-optimality gaps can be established: since regret in the risk-neutral setting is defined as $\widetilde{\calR}(K) \defeq 
\sumk 
(\widetilde{V}_1^* - \widetilde{V}_1^{\pi^k})(s_1^k)$, we have $\widetilde{\calR}(K) = 
\sumk 
D^{\pi^k}_1 (s^k_1)$ (with $D^{\pi^k}_h$ as implied in \cref{fact:risk-neut-gap-meet-cond}). In words, the regret can be written as the sum of $\{D^{\pi^k}_1\}$ defined for the Bellman difference condition, in which sub-optimality gaps take the role of rewards.

The  proof of \cref{fact:risk-neut-gap-meet-cond} exploits the Bellman equations under the risk-neutral setting and, in particular, the linearity of $\widetilde{Q}^{\pi}_h$ in terms of $r_h$ and $\widetilde{V}^{\pi}_{h+1}$. 
However, such linear properties are not available in risk-sensitive \gls*{RL}, as seen in \cref{eqn:bellman}, where the non-linearity is induced by the entropic risk measure 
$X\mapsto\frac{1}{\beta}\log(\expect[e^{\beta X}])$. This suggests that a simple definition of sub-optimality
gaps such as \cref{eq:gap_risk_neut} may not be appropriate, and an alternative definition is  necessary.

\subsection{Cascading Structure}
To introduce sub-optimality gaps for risk-sensitive RL, we need a few additional notations. 
We denote by $\tau$ a trajectory of length $H$, which is a series of state-action pairs $\{(s_j,a_j)\}_{j\in[H]}$, and we let $\calT$ be the set of all possible trajectories. For any trajectory $\tau \in \calT$ and $h \in [H]$, we let $\tau_h$ denote the trajectory that consists of the first $h$ elements in $\tau$, and we define the set $\calT_h \defeq \{\tau_h: \tau\in\calT\}$. Note that  $\tau_H = \tau$ and $\calT_H = \calT$. We also let $\tau_0$ be an empty trajectory and $\calT_0 \defeq \emptyset$. We further define the cumulative reward function $R$ on trajectories such that $R(\tau_0) \defeq 0$ and $R(\tau_h) \defeq \sum_{j\in[h]} r_j(s_j,a_j)$ for $h \in [H]$ and 
$\tau \in \calT$.

Motivated by the discussion in \cref{sec:bellman-diff-cond}, we propose the following definition of sub-optimality gaps for risk-sensitive \gls*{RL}:
for any step $h$ and trajectory
$\trj$,
we let 
\ifDoubleColumn
\begin{align}
    & \Delta_{h,\beta}(s,a;\trj_{h-1}) \nonumber\\ 
    & \quad \coloneqq
    \psi_{\beta}\cdot e^{\beta\cdot R(\tau_{h-1})}\cdot[e^{\beta\cdot V_{h}^{*}(s)}-e^{\beta\cdot Q_{h}^{*}(s,a)}],
\label{eq:gap}
\end{align}
\else
\begin{align}
    \Delta_{h,\beta}(s,a;\trj_{h-1})\coloneqq
    \psi_{\beta}\cdot e^{\beta\cdot R(\tau_{h-1})}\cdot[e^{\beta\cdot V_{h}^{*}(s)}-e^{\beta\cdot Q_{h}^{*}(s,a)}],
\label{eq:gap}
\end{align}
\fi
where $\psi_{\beta} \defeq 1/\beta$ for $\beta > 0$ and  $\psi_{\beta} \defeq e^{-\beta H} / 
\beta$ for $\beta < 0$. It is not hard to see that $\Delta_{h,\beta} \ge 0$ for any $\beta\neq 0$.

Let us remark on several noteworthy properties of this gap definition. First, in contrast with $\widetilde{\Delta}_h$ defined in \cref{eq:gap_risk_neut} for the risk-neutral setting, which only depends on $\pi^*$ and a single state-action pair $(s,a)$ at step $h$, the gap $\Delta_{h,\beta}$ defined in \cref{eq:gap} additionally depends on the trajectory prior to step $h$. Specifically, 
the gap consists of two components: the factor $e^{\beta\cdot R(\tau_{h-1})}$,
which is with respect to an uncontrolled 
trajectory $\tau_{h-1}$
up to step $h-1$, as well as a quantity $e^{\beta\cdot V_{h}^{*}(s)}-e^{\beta\cdot Q_{h}^{*}(s,a)}$,
which is with respect to the trajectory controlled by an optimal policy
$\pi^{*}$ starting from step $h$ and state-action pair $(s,a)$ 
(actions
in the expectation of $V^*_h$ and $Q^*_h$ follow $\pi^{*}$ after step $h$).
This means that $\Delta_{h,\beta}$ contains both uncontrolled and optimally controlled trajectories. 
Second, given a trajectory $\trj$ and
for $\beta > 0$, as $h$ increases, the multiplicative
factor $e^{\beta\cdot R(\tau_{h-1})}\in[1,e^{\beta(h-1)}]$ is non-decreasing in $h$
and the exponential value functions $e^{\beta\cdot V_{h}^{*}(s)}$, $e^{\beta\cdot Q_{h}^{*}(s,a)}\in[1,e^{\beta(H-h+1)}]$
are non-increasing in $h$; vice versa for $\beta<0$. See \cref{fig:cascaded_gap} for an illustration of this property. 
In view of their special
structure, we name these gaps as \emph{cascaded gaps}.

\ifDoubleColumn
    \begin{figure*}
    \centering
    \begin{tikzpicture}
    
    \node[text width=3cm] at (3,2) 
        {Cascaded Gaps};
    
    \filldraw[fill=blue!20] (-1,1) rectangle +(1,1/2);
    \filldraw[fill=blue!20] (0,1) rectangle +(1,1/2);
    \filldraw[fill=blue!20] (1,1) rectangle +(1,1/2);
    
    \node[text width=1cm] at (-2,1.25) 
        {$h=1$};
        
    \node[text width=6cm] at (5.5,1.25) 
        {$\psi_\beta\cdot e^{\beta\cdot R(\tau_{0})}[e^{\beta\cdot V_{1}^{*}(s)}-e^{\beta\cdot Q_{1}^*(s,a)}]$};
    \draw[-{Latex[open]}] (0.5,0.95) to [out=-15,in=195] (5.7,0.95);
    
    \filldraw[fill=red!20,rounded corners] (-1,0) rectangle +(1,1/2);
    \filldraw[fill=blue!20] (0,0) rectangle +(1,1/2);
    \filldraw[fill=blue!20] (1,0) rectangle +(1,1/2);
    
    \node[text width=1cm] at (-2,0.25) 
        {$h=2$};
    
    \node[text width=6cm] at (5.5,0.25) 
        {$\psi_\beta\cdot e^{\beta\cdot R(\tau_{1})}[e^{\beta\cdot V_{2}^{*}(s)}-e^{\beta\cdot Q_{2}^*(s,a)}]$};
    \draw[-{Latex[open]}] (-0.5,-0.05) to [out=-15,in=195] (3.7,-0.05);
    \draw[-{Latex[open]}] (1,-0.05) to [out=-15,in=195] (5.7,-0.05);
    
    \filldraw[fill=red!20,rounded corners] (-1,-1) rectangle +(1,1/2);
    \filldraw[fill=red!20,rounded corners] (0,-1) rectangle +(1,1/2);
    \filldraw[fill=blue!20] (1,-1) rectangle +(1,1/2);
    
    \node[text width=1cm] at (-2,-0.75) 
        {$h=3$};
    
    \node[text width=6cm] at (5.5,-0.75) 
        {$\psi_\beta\cdot e^{\beta\cdot R(\tau_{2})}[e^{\beta\cdot V_{3}^{*}(s)}-e^{\beta\cdot Q_{3}^*(s,a)}]$};
    \draw[-{Latex[open]}] (0,-1.05) to [out=-15,in=195] (3.7,-1.05);
    \draw[-{Latex[open]}] (1.5,-1.05) to [out=-15,in=195] (5.7,-1.05);
    
    \node[text width=3cm] at (11.5,2) 
        {Risk-Neutral Gaps};
        
    \filldraw[fill=blue!20] (8,1) rectangle +(1,1/2);
    \filldraw[fill=blue!20] (9,1) rectangle +(1,1/2);
    \filldraw[fill=blue!20] (10,1) rectangle +(1,1/2);
    
    \node[text width=2.5cm] at (13,1.25) 
        {$\widetilde V_{1}^{*}(s)-\widetilde Q_{1}^*(s,a)$};
    \draw[-{Latex[open]}] (9.5,0.95) to [out=-48,in=190] (11.7,1.25);
    
    \filldraw[fill=white!20,dashed] (8,0) rectangle +(1,1/2);
    \filldraw[fill=blue!20] (9,0) rectangle +(1,1/2);
    \filldraw[fill=blue!20] (10,0) rectangle +(1,1/2);
    
    \node[text width=2.5cm] at (13,0.25) 
        {$\widetilde V_{2}^{*}(s)-\widetilde Q_{2}^*(s,a)$};
    \draw[-{Latex[open]}] (10,-0.05) to [out=-45,in=185] (11.7,0.25);
    
    \filldraw[fill=white!20,dashed] (8,-1) rectangle +(1,1/2);
    \filldraw[fill=white!20,dashed] (9,-1) rectangle +(1,1/2);
    \filldraw[fill=blue!20] (10,-1) rectangle +(1,1/2);
    
    \node[text width=2.5cm] at (13,-0.75) 
        {$\widetilde V_{3}^{*}(s)-\widetilde Q_{3}^*(s,a)$};
    \draw[-{Latex[open]}] (10.5,-1.05) to [out=-50,in=185] (11.7,-0.75);
    
    \end{tikzpicture}
    \caption{A comparison of the cascaded gaps \eqref{eq:gap} in the risk-sensitive setting ($\beta > 0$) and risk-neutral gaps \eqref{eq:gap_risk_neut} for $H = 3$. 
    The blue blocks illustrate $\pi^*$-controlled trajectories, whereas the red blocks illustrate uncontrolled trajectories. Note that for the top cascaded gap,
    the uncontrolled trajectory part $e^{\beta \cdot R(\trj_0) } = 1$  since  $R(\tau_0) = 0$ by definition. 
    \label{fig:cascaded_gap}}
    \end{figure*}
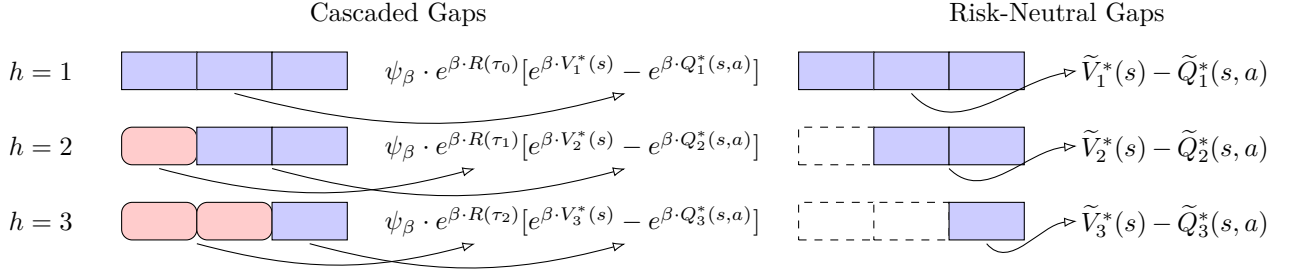
\else
    \begin{figure*}[h]
    \centering
    \begin{tikzpicture}
    
    \node[text width=3cm] at (3,2) 
        {Cascaded Gaps};
    
    \filldraw[fill=blue!20] (-1,1) rectangle +(1,1/2);
    \filldraw[fill=blue!20] (0,1) rectangle +(1,1/2);
    \filldraw[fill=blue!20] (1,1) rectangle +(1,1/2);
    
    \node[text width=1cm] at (-2,1.25) 
        {$h=1$};
        
    \node[text width=6cm] at (5.5,1.25) 
        {$\psi_\beta\cdot e^{\beta\cdot R(\tau_{0})}[e^{\beta\cdot V_{1}^{*}(s)}-e^{\beta\cdot Q_{1}^*(s,a)}]$};
    \draw[-{Latex[open]}] (0.5,0.95) to [out=-15,in=195] (5.7,0.95);
    
    \filldraw[fill=red!20,rounded corners] (-1,0) rectangle +(1,1/2);
    \filldraw[fill=blue!20] (0,0) rectangle +(1,1/2);
    \filldraw[fill=blue!20] (1,0) rectangle +(1,1/2);
    
    \node[text width=1cm] at (-2,0.25) 
        {$h=2$};
    
    \node[text width=6cm] at (5.5,0.25) 
        {$\psi_\beta\cdot e^{\beta\cdot R(\tau_{1})}[e^{\beta\cdot V_{2}^{*}(s)}-e^{\beta\cdot Q_{2}^*(s,a)}]$};
    \draw[-{Latex[open]}] (-0.5,-0.05) to [out=-15,in=195] (3.7,-0.05);
    \draw[-{Latex[open]}] (1,-0.05) to [out=-15,in=195] (5.7,-0.05);
    
    \filldraw[fill=red!20,rounded corners] (-1,-1) rectangle +(1,1/2);
    \filldraw[fill=red!20,rounded corners] (0,-1) rectangle +(1,1/2);
    \filldraw[fill=blue!20] (1,-1) rectangle +(1,1/2);
    
    \node[text width=1cm] at (-2,-0.75) 
        {$h=3$};
    
    \node[text width=6cm] at (5.5,-0.75) 
        {$\psi_\beta\cdot e^{\beta\cdot R(\tau_{2})}[e^{\beta\cdot V_{3}^{*}(s)}-e^{\beta\cdot Q_{3}^*(s,a)}]$};
    \draw[-{Latex[open]}] (0,-1.05) to [out=-15,in=195] (3.7,-1.05);
    \draw[-{Latex[open]}] (1.5,-1.05) to [out=-15,in=195] (5.7,-1.05);
    
    \node[text width=3cm] at (11.5,2) 
        {Risk-Neutral Gaps};
        
    \filldraw[fill=blue!20] (8,1) rectangle +(1,1/2);
    \filldraw[fill=blue!20] (9,1) rectangle +(1,1/2);
    \filldraw[fill=blue!20] (10,1) rectangle +(1,1/2);
    
    \node[text width=2.5cm] at (13,1.25) 
        {$\widetilde V_{1}^{*}(s)-\widetilde Q_{1}^*(s,a)$};
    \draw[-{Latex[open]}] (9.5,0.95) to [out=-48,in=190] (11.7,1.25);
    
    \filldraw[fill=white!20,dashed] (8,0) rectangle +(1,1/2);
    \filldraw[fill=blue!20] (9,0) rectangle +(1,1/2);
    \filldraw[fill=blue!20] (10,0) rectangle +(1,1/2);
    
    \node[text width=2.5cm] at (13,0.25) 
        {$\widetilde V_{2}^{*}(s)-\widetilde Q_{2}^*(s,a)$};
    \draw[-{Latex[open]}] (10,-0.05) to [out=-45,in=185] (11.7,0.25);
    
    \filldraw[fill=white!20,dashed] (8,-1) rectangle +(1,1/2);
    \filldraw[fill=white!20,dashed] (9,-1) rectangle +(1,1/2);
    \filldraw[fill=blue!20] (10,-1) rectangle +(1,1/2);
    
    \node[text width=2.5cm] at (13,-0.75) 
        {$\widetilde V_{3}^{*}(s)-\widetilde Q_{3}^*(s,a)$};
    \draw[-{Latex[open]}] (10.5,-1.05) to [out=-50,in=185] (11.7,-0.75);
    
    \end{tikzpicture}
    \caption{A comparison of the cascaded gaps \eqref{eq:gap} in the risk-sensitive setting ($\beta > 0$) and risk-neutral gaps \eqref{eq:gap_risk_neut} for $H = 3$. 
    The blue blocks illustrate $\pi^*$-controlled trajectories, whereas the red blocks illustrate uncontrolled trajectories. Note that for the top cascaded gap,
    the uncontrolled trajectory part $e^{\beta \cdot R(\trj_0) } = 1$  since  $R(\tau_0) = 0$ by definition. 
    \label{fig:cascaded_gap}}
    \end{figure*}
\fi

We will soon discuss the factor $\psi_\beta$,  another distinctive and important feature of cascaded gaps, but for now let us  show that 
the gaps
satisfy \cref{cond:bellman_diff}.

\begin{fact}
For any $\beta\ne0$, we have that $\{\Delta_{h,\beta}\}_{h\in[H]}$ satisfy \cref{cond:bellman_diff} with $Z^{\pi}_h \defeq e^{\beta (R(\trj_{h-1}) + V^{\pi}_h(s))}$.
\label{fact:risk-sen-gaps-meet-cond}
\end{fact}
\begin{proof}
Let us consider an arbitrary policy $\pi$ and fix a tuple $(h,s,a)$ such that $a = \pi_h(s)$. We also fix a trajectory $\trj$ whose $h$-th element is $(s,a)$.
We have 
\ifDoubleColumn
    \begin{align*}
        & \Delta_{h,\beta}(s,a;\trj_{h-1}) \\
        &=e^{\beta\cdot R(\tau_{h-1})}[e^{\beta\cdot V_{h}^{*}(s)}-e^{\beta\cdot Q_{h}^{*}(s,a)}]\\
         & =e^{\beta\cdot R(\tau_{h-1})}[e^{\beta\cdot V_{h}^{*}(s)}-e^{\beta\cdot V_{h}^{\pi}(s)}] \\
         & \quad +e^{\beta\cdot R(\tau_{h-1})} [e^{\beta\cdot V_{h}^{\pi}(s)}-e^{\beta\cdot Q_{h}^{*}(s,a)}]\\
         & =e^{\beta\cdot R(\tau_{h-1})}[e^{\beta\cdot V_{h}^{*}(s)}-e^{\beta\cdot V_{h}^{\pi}(s)}] \\ 
         & \quad + e^{\beta\cdot R(\tau_{h-1})} [e^{\beta\cdot Q_{h}^{\pi}(s,a)}-e^{\beta\cdot Q_{h}^{*}(s,a)}]\\
         & \labelrel={eqn:bellman-decomp2}e^{\beta\cdot R(\tau_{h-1})}[e^{\beta\cdot V_{h}^{*}(s)}-e^{\beta\cdot V_{h}^{\pi}(s)}] \\
         & \quad +e^{\beta\cdot R(\tau_{h-1})}\left[e^{\beta\cdot r_{h}(s,a)}\E_{s'\sim \calP_h(\cdot|s,a)}[e^{\beta\cdot V_{h+1}^{\pi}(s')}]\right]\\
         & \quad-e^{\beta\cdot R(\tau_{h-1})}\left[e^{\beta\cdot r_{h}(s,a)}\E_{s'\sim \calP_h(\cdot|s,a)}[e^{\beta\cdot V_{h+1}^{*}(s')}]\right]\\
         & =e^{\beta\cdot R(\tau_{h-1})}[e^{\beta\cdot V_{h}^{*}(s)}-e^{\beta\cdot V_{h}^{\pi}(s)}] \\
         & \quad -\E_{s'\sim \calP_h(\cdot|s,a)}[e^{\beta\cdot R(\tau_{h})}(e^{\beta\cdot V_{h+1}^{*}(s')}-e^{\beta\cdot V_{h+1}^{\pi}(s')})]\\
         & =\vdiff^{\pi}_{h}(s)-\E_{s'\sim \calP_h(\cdot|s,a)}[\vdiff^{\pi}_{h+1}(s')],
    \end{align*}
\else
    \begin{align*}
        \Delta_{h,\beta}(s,a;\trj_{h-1}) & =e^{\beta\cdot R(\tau_{h-1})}[e^{\beta\cdot V_{h}^{*}(s)}-e^{\beta\cdot Q_{h}^{*}(s,a)}]\\
         & =e^{\beta\cdot R(\tau_{h-1})}[e^{\beta\cdot V_{h}^{*}(s)}-e^{\beta\cdot V_{h}^{\pi}(s)} +e^{\beta\cdot V_{h}^{\pi}(s)}-e^{\beta\cdot Q_{h}^{*}(s,a)}]\\
         & =e^{\beta\cdot R(\tau_{h-1})}[e^{\beta\cdot V_{h}^{*}(s)}-e^{\beta\cdot V_{h}^{\pi}(s)} +e^{\beta\cdot Q_{h}^{\pi}(s,a)}-e^{\beta\cdot Q_{h}^{*}(s,a)}]\\
         & \labelrel={eqn:bellman-decomp2}e^{\beta\cdot R(\tau_{h-1})}[e^{\beta\cdot V_{h}^{*}(s)}-e^{\beta\cdot V_{h}^{\pi}(s)}] +e^{\beta\cdot R(\tau_{h-1})}\left[e^{\beta\cdot r_{h}(s,a)}\E_{s'\sim \calP_h(\cdot|s,a)}[e^{\beta\cdot V_{h+1}^{\pi}(s')}]\right]\\
         & \quad-e^{\beta\cdot R(\tau_{h-1})}\left[e^{\beta\cdot r_{h}(s,a)}\E_{s'\sim \calP_h(\cdot|s,a)}[e^{\beta\cdot V_{h+1}^{*}(s')}]\right]\\
         & =e^{\beta\cdot R(\tau_{h-1})}[e^{\beta\cdot V_{h}^{*}(s)}-e^{\beta\cdot V_{h}^{\pi}(s)}] -\E_{s'\sim \calP_h(\cdot|s,a)}[e^{\beta\cdot R(\tau_{h})}(e^{\beta\cdot V_{h+1}^{*}(s')}-e^{\beta\cdot V_{h+1}^{\pi}(s')})]\\
         & =\vdiff^{\pi}_{h}(s)-\E_{s'\sim \calP_h(\cdot|s,a)}[\vdiff^{\pi}_{h+1}(s')],
    \end{align*}
\fi
where step \eqref{eqn:bellman-decomp2} holds by taking exponential on both sides of the Bellman equation \eqref{eqn:bellman}, and the last step holds by the definition of $\vdiff^{\pi}_{h}$ in \cref{cond:bellman_diff} and that of $Z^{\pi}_h$.
\end{proof}
The proof crucially exploits the multiplicative property of the Bellman equation \eqref{eqn:bellman} raised to exponential\footnote{The result of the transformation  is known as the exponential Bellman equation  \citep{fei2021exponential}.} as well as the cascading structure of 
$\Delta_{h,\beta}$. 
Based on $\Delta_{h,\beta}$, 
we define the \emph{minimal cascaded gap} $\Delta_{\mmin,\beta}$ as the minimum non-zero cascaded gap over tuples $(h,s,a)\in [H]\times\calS\times\calA$ and trajectories $\trj\in\calT$, \ie,
\ifDoubleColumn
\begin{align}
    & \Delta_{\mmin,\beta} \nonumber \\ 
    & \quad
    \defeq \min_{h,s,a,\trj} \{
    \Delta_{h,\beta}(s,a;\trj_{h-1}) \st \Delta_{h,\beta}(s,a;\trj_{h-1})\neq 0
\},
\label{eq:gap_min}
\end{align}
\else
\begin{align}
    \Delta_{\mmin,\beta} \defeq \min_{h,s,a,\trj} \{
    \Delta_{h,\beta}(s,a;\trj_{h-1}) \st \Delta_{h,\beta}(s,a;\trj_{h-1})\neq 0
\},
\label{eq:gap_min}
\end{align}
\fi
For any fixed $\beta$, the minimal gap serves as a measure for the difficulty of the corresponding \gls*{MDP} problem. We assume $\Delta_{\mmin,\beta} > 0$ throughout the paper to avoid triviality.

\subsection{Normalization for Risk Consistency}
One might notice that our notion of cascaded gaps is not the only gap definition that satisfies \cref{cond:bellman_diff}. Indeed, another candidate for the gap definition  would be $\Delta'_{h,\beta}(s,a) \defeq \sign(\beta) \cdot e^{\beta\cdot R(\tau_{h-1})} [e^{\beta\cdot V_{h}^{*}(s)}-e^{\beta\cdot Q_{h}^{*}(s,a)}]$, with the only difference, compared to $\Delta_{h,\beta}$, being that it replaces the normalization factor $\psi_{\beta}$ with the sign of $\beta$. It is not hard to show that this alternative definition also meets \cref{cond:bellman_diff}.

Yet, we demonstrate that the normalizer $\psi_\beta$ is crucial for the gap  $\Delta_{h,\beta}$ to showcase \textit{risk consistency}: the gap has the same order of magnitude when $|\beta|$ is fixed and recovers the risk neutral gap $\widetilde{\Delta}_h$ as $|\beta|\to 0$.
To illustrate this point
(as well as the deficiency of the alternative $\Delta'_{h,\beta}$), 
let us consider an \gls*{MDP} with arbitrary transition kernels and its reward function satisfying $r_h(s,a) = 1$ for $(h,s,a)\in [H-1]\times\cS\times\cA$, $r_H(s,a^*) = 1$ for some action $a^* \in \cA$, and $r_H(s,a) = 0$ for $\cA \setminus  \{a^*\}$. That is, this \gls*{MDP} has all its rewards equal to $1$ 
except for the last step when sub-optimal actions are taken (which yields zero rewards).

For   $\beta > 0$, the alternative gap $\Delta'_{h,\beta}$ of the above \gls*{MDP} is on the order of $e^{\beta H} - 1$ (which  grows exponentially in $\beta$), but for $\beta < 0$, its order is of $1 - e^{\beta H}$ (which is upper bounded by 1 for any $\beta < 0$). 
Therefore, 
the magnitude of $\Delta'_{h,\beta}$ is inconsistent under different signs of $\beta$. 
On the other hand, it can be verified that our definition $\Delta_{h,\beta}$ is on the same order of $(e^{|\beta| H} - 1)/|\beta|$ for \textit{all} $\beta \ne 0$, thanks to the risk-dependent normalization factor $\psi_{\beta}$. In addition, as $\beta \to 0$, we have $\Delta_{h,\beta}(s,a; \trj_{h-1}) \to V^*_h(s) - Q^*_h(s,a) = \widetilde{\Delta}_h(s,a)$ for any $(h, s, a, \trj)$ by L'Hospital's rule, thereby recovering the definition of sub-optimality gaps in the risk-neutral setting; nevertheless, $\Delta'_{h,\beta}$ tends to 0 and becomes degenerate as $\beta\to 0$.

\section{Algorithms}

    We consider two model-free algorithms for risk-sensitive \gls*{RL},  \textsc{RSVI2} (\cref{alg:rsv2}) and \textsc{RSQ2} (\cref{alg:rsq2}), both of which are proposed in \citet{fei2021exponential}. 
    
    \cref{alg:rsv2} is based on value iteration that features an optimistic estimate $Q_h$ of the action value with a bonus term.
    In episode $k$, we compute at each step $h$ the sample average 
    \ifDoubleColumn
        \begin{align}\label{eqn:rsv2-update-w}
            \begin{split}
                w_h(s,a) & \leftarrow \frac{1}{N_{h}(s,a)}\sum_{i\in[k-1]}\II\{(s_h^i,a_h^i) = (s,a)\} \\
                & \qquad \cdot e^{\beta[r_h(s,a)+V_{h+1}(s_{h+1}^i)]}
            \end{split}
        \end{align}
    \else
        \begin{align}\label{eqn:rsv2-update-w}
            \begin{split}
                w_h(s,a) & \leftarrow \frac{1}{N_{h}(s,a)}\sum_{i\in[k-1]}\II\{(s_h^i,a_h^i) = (s,a)\} \cdot e^{\beta[r_h(s,a)+V_{h+1}(s_{h+1}^i)]}
            \end{split}
        \end{align}
    \fi
    over prior episodes
    for all visited state-action pairs $(s,a)$.
    The bonus is given by 
    \begin{align}\label{eqn:rsv2-update-b}
        b_{h}(s,a)\leftarrow c\left|e^{\beta (H-h+1)}-1\right|\sqrt{\frac{S\log(2SAHK/\delta)}{N_{h}(s,a)}},
    \end{align}
    where $c > 0$ is a universal constant. It  decays in both step $h$ and the  number of visits $N_h$, thus also known as the doubly decaying bonus \citep{fei2021exponential}, and enforces the principle of \emph{Risk-Sensitive Optimism in the Face of Uncertainty} that encourages more exploration of  less frequently visited state-action pairs. 
    We then compute the optimistic estimate of the action-value function through
    \begin{align}\label{eqn:rsv2-update-q}
        Q_h(s,a)\gets\frac{1}{\beta}\log(G_h(s,a)),
    \end{align}
    where
    \ifDoubleColumn
        \begin{align*}
            &G_{h}(s,a) \nonumber \\
            &\leftarrow
            \begin{cases}
        		\min\{e^{\beta(H-h+1)},w_{h}(s,a)+b_{h}(s,a)\}, & \text{if }\beta>0;\\
        		\max\{e^{\beta(H-h+1)},w_{h}(s,a)-b_{h}(s,a)\}, & \text{if }\beta<0.
        	\end{cases}
        \end{align*}
    \else
        \begin{align*}
            G_{h}(s,a)\leftarrow
            \begin{cases}
        		\min\{e^{\beta(H-h+1)},w_{h}(s,a)+b_{h}(s,a)\}, & \text{if }\beta>0;\\
        		\max\{e^{\beta(H-h+1)},w_{h}(s,a)-b_{h}(s,a)\}, & \text{if }\beta<0.
        	\end{cases}
        \end{align*}
    \fi
    Note that for $\beta>0$, the addition of the bonus term $b_h$ represents optimism in risk-seeking decision making, whereas for $\beta<0$ the subtraction of the bonus term corresponds to  optimism in risk-averse decision making.
    Finally, in the policy execution stage, action $a_h$ is taken following the policy that maximizes $Q_h(s_h,\cdot)$ over $\calA$.
    
    \ifDoubleColumn
    \begin{algorithm}[h]
    	\begin{algorithmic}[1]
    		
    		\Require number of episodes $K$, confidence level
    		$\delta\in(0,1]$, and risk parameter $\beta\ne0$
    		
    		\State $Q_{h}(s,a),V_h(s)\leftarrow H-h+1$, $w_{h}(s,a)\leftarrow 0$, and
    		\Statex $N_h(s,a) \leftarrow 0$ for all $(h,s,a)\in[H+1]\times\calS\times\calA$
    		
    		\For{episode $k=1,\ldots,K$}
    		
    		\For{step $h=H,\ldots,1$}\label{line:rsv2_estim_value_begin}
    		
    		\For{$(s,a)\in\calS\times\calA$ such that $N_{h}(s,a)\ge1$}
    		
            \State Update $w_h(s,a)$ following \eqref{eqn:rsv2-update-w}
    		
            \State Update $b_h(s,a)$ following \eqref{eqn:rsv2-update-b}
    		
            \State Update $Q_h(s,a)$ following \eqref{eqn:rsv2-update-q}
            
    		
    		\State $V_{h}(s)\leftarrow\max_{a'\in\calA}Q_{h}(s,a')$ \label{line:rsv2_value}
    		
    		\EndFor 
    		
    		\EndFor \label{line:rsv2_estim_value_end}
    		
    		\For{step $h=1,\ldots,H$}
    		\label{line:rsv2_exec_policy_begin}
    		
    		\State Take action $a_{h}\leftarrow\argmax_{a\in\calA}Q_{h}(s_{h},a)$
    			and 
    		\Statex[3] observe $r_{h}(s_{h},a_{h})$ and $s_{h+1}$
    		
    		\State $N_{h}(s_{h},a_{h})\leftarrow N_{h}(s_{h},a_{h})+1$
    		
    		\EndFor \label{line:rsv2_exec_policy_end}
    		
    		\EndFor
    	\end{algorithmic}
    	\caption{\textsc{RSVI2} 
    	\label{alg:rsv2}}
    \end{algorithm}
    \else
    \begin{algorithm}[h]
    	\begin{algorithmic}[1]
    		
    		\Require number of episodes $K$, confidence level
    		$\delta\in(0,1]$, and risk parameter $\beta\ne0$
    		
    		\State $Q_{h}(s,a),V_h(s)\leftarrow H-h+1$, $w_{h}(s,a)\leftarrow 0$, and $N_h(s,a) \leftarrow 0$ for all $(h,s,a)\in[H+1]\times\calS\times\calA$
    		
    		\For{episode $k=1,\ldots,K$}
    		
    		\For{step $h=H,\ldots,1$}\label{line:rsv2_estim_value_begin}
    		
    		\For{ $(s,a)\in\calS\times\calA$ such that $N_{h}(s,a)\ge1$}
    		
            \State Update $w_h(s,a)$ following \eqref{eqn:rsv2-update-w}
    		
            \State Update $b_h(s,a)$ following \eqref{eqn:rsv2-update-b}
    		
            \State Update $Q_h(s,a)$ following \eqref{eqn:rsv2-update-q}
            
    		
    		\State $V_{h}(s)\leftarrow\max_{a'\in\calA}Q_{h}(s,a')$ \label{line:rsv2_value}
    		
    		\EndFor 
    		
    		\EndFor \label{line:rsv2_estim_value_end}
    		
    		\For{ step $h=1,\ldots,H$}
    		\label{line:rsv2_exec_policy_begin}
    		
    		\State Take action $a_{h}\leftarrow\argmax_{a\in\calA}Q_{h}(s_{h},a)$
    			and observe $r_{h}(s_{h},a_{h})$ and $s_{h+1}$
    		
    		\State $N_{h}(s_{h},a_{h})\leftarrow N_{h}(s_{h},a_{h})+1$
    		
    		\EndFor \label{line:rsv2_exec_policy_end}
    		
    		\EndFor
    	\end{algorithmic}
    	\caption{\textsc{RSVI2} 
    	\label{alg:rsv2}}
    \end{algorithm}
    \fi
    
    On the other hand, \cref{alg:rsq2} follows the paradigm of Q-learning. In step $h$ it  computes the (exponential) moving average estimate
    \ifDoubleColumn
        \begin{align}\label{eqn:rsq2-update-w}
            \begin{split}
                w_{h}(s_{h},a_{h}) & \leftarrow(1-\alpha_t) G_{h}(s_{h},a_{h}) \\
                & \quad +\alpha_{t}\cdot e^{\beta[r_{h}(s_{h},a_{h})+V_{h+1}(s_{h+1})]}
            \end{split}
        \end{align}
    \else
        \begin{align}\label{eqn:rsq2-update-w}
            w_{h}(s_{h},a_{h}) & \leftarrow(1-\alpha_t) G_{h}(s_{h},a_{h})+\alpha_{t}\cdot e^{\beta[r_{h}(s_{h},a_{h})+V_{h+1}(s_{h+1})]}
        \end{align}
    \fi
    through online updates instead of  batch updates as used in \cref{alg:rsv2}. However, it uses a similar doubly decaying bonus term 
    \begin{align}\label{eqn:qlearn_bonus_def}
        b_{h,t}\leftarrow c\left|e^{\beta (H-h+1)}-1\right|\sqrt{\frac{H\log(2SAHK/\delta)}{t}}
    \end{align}
    for some universal constant $c > 0$, in enforcing optimism for efficient exploration. Similarly, the optimistic estimation of the value function is set as   
    \begin{align}\label{eqn:rsq2-update-q}
        Q_h(s_h,a_h)\gets \frac{1}{\beta}\log(G_h(s_h,a_h)),
    \end{align}
    where the update on the exponential value function and truncation are given by
    \ifDoubleColumn
        \begin{align*}
            & G_{h}(s_{h},a_{h}) \nonumber \\
            &\leftarrow\begin{cases}
        		\min\{e^{\beta(H-h+1)},w_{h}(s_{h},a_{h})+\alpha_{t}b_{t}\}, & \text{if }\beta>0;\\
        		\max\{e^{\beta(H-h+1)},w_{h}(s_{h},a_{h})-\alpha_{t}b_{t}\}, & \text{if }\beta<0.
        		\end{cases}
        \end{align*}
    \else
        \begin{align*}
            & G_{h}(s_{h},a_{h}) \leftarrow\begin{cases}
        		\min\{e^{\beta(H-h+1)},w_{h}(s_{h},a_{h})+\alpha_{t}b_{t}\}, & \text{if }\beta>0;\\
        		\max\{e^{\beta(H-h+1)},w_{h}(s_{h},a_{h})-\alpha_{t}b_{t}\}, & \text{if }\beta<0.
        		\end{cases}
        \end{align*}
    \fi
    
    \ifDoubleColumn
    \begin{algorithm}[t]
    	\begin{algorithmic}[1]
    		
    		\Require number of episodes $K$, confidence level
    		$\delta\in(0,1]$, and risk parameter
    		$\beta\ne0$
    		
    		\State $Q_{h}(s,a), V_{h}(s) \leftarrow H-h+1$ if $\beta>0$; 
    		\Statex $Q_{h}(s), V_{h}(s,a) \leftarrow 0$ if $\beta<0$,
    		for all $(h,s,a)\in$ 
    		\Statex $[H+1]\times\calS\times\calA$
    		
    		\State $N_{h}(s,a)\leftarrow0$ for all $(h,s,a)\in[H]\times\calS\times\calA$, and 
    		\Statex $\alpha_t \leftarrow \frac{H+1}{H+t}$ for all $t\in\ZZ_{+}$
    		
    		\For{episode $k=1,\ldots,K$}
    		
    		\State Receive the initial state $s_{1}$
    		
    		\For{step $h=1,\ldots,H$} \label{line:qlearn_estim_value_begin}
    		
    		\State Take action $a_{h}\leftarrow\argmax_{a'\in\calA}Q_{h}(s_{h},a')$, 
    		\Statex[3] 
    		and observe $r_{h}(s_{h},a_{h})$ and $s_{h+1}$ 
    		\label{line:qlearn_exec_policy}
    		
    		\State $N_{h}(s_{h},a_{h})\leftarrow N_{h}(s_{h},a_{h})+1$
    		
    			\State $t \leftarrow N_{h}(s_{h},a_{h})$

            \State Update $w_h(s_h,a_h)$ following \eqref{eqn:rsq2-update-w}
    		
    		\State Update $b_{h,t}$ following \eqref{eqn:qlearn_bonus_def} 
            \State Update $Q_h(s_h,a_h)$ following \eqref{eqn:rsq2-update-q}
            
    		
    		\State $V_{h}(s_{h})\leftarrow\max_{a'\in\calA}Q_{h}(s_{h},a')$\label{line:qlearn_value}
    		
    		\EndFor \label{line:qlearn_estim_value_end}
    		
    		\EndFor
    		
    	\end{algorithmic}
    	\caption{\textsc{RSQ2} 
    	\label{alg:rsq2}}
    \end{algorithm}
    \else
    \begin{algorithm}[h]
    	\begin{algorithmic}[1]
    		
    		\Require number of episodes $K$, confidence level
    		$\delta\in(0,1]$, and risk parameter
    		$\beta\ne0$
    		
    		\State $Q_{h}(s,a), V_{h}(s) \leftarrow H-h+1$ if $\beta>0$; $Q_{h}(s), V_{h}(s,a) \leftarrow 0$ if $\beta<0$,
    		for all $(h,s,a)\in[H+1]\times\calS\times\calA$
    		
    		\State $N_{h}(s,a)\leftarrow0$ for all $(h,s,a)\in[H]\times\calS\times\calA$, and $\alpha_t \leftarrow \frac{H+1}{H+t}$ for all $t\in\ZZ$
    		
    		\For{episode $k=1,\ldots,K$}
    		
    		\State Receive the initial state $s_{1}$
    		
    		\For{step $h=1,\ldots,H$} \label{line:qlearn_estim_value_begin}
    		
    		\State Take action $a_{h}\leftarrow\argmax_{a'\in\calA}Q_{h}(s_{h},a')$,
    		and observe $r_{h}(s_{h},a_{h})$ and $s_{h+1}$ 
    		\label{line:qlearn_exec_policy}
    		
    		\State $t=N_{h}(s_{h},a_{h})\leftarrow N_{h}(s_{h},a_{h})+1$

            \State Update $w_h(s_h,a_h)$ following \eqref{eqn:rsq2-update-w}
            
               		\State Update $b_{h,t}$ following \eqref{eqn:qlearn_bonus_def} 
    		
            \State Update $Q_h(s_h,a_h)$ following \eqref{eqn:rsq2-update-q}
            
    		
    		\State $V_{h}(s_{h})\leftarrow\max_{a'\in\calA}Q_{h}(s_{h},a')$\label{line:qlearn_value}
    		
    		\EndFor \label{line:qlearn_estim_value_end}
    		
    		\EndFor
    		
    	\end{algorithmic}
    	\caption{\textsc{RSQ2} 
    	\label{alg:rsq2}}
    \end{algorithm}
    \fi

\section{Main Results}

In this section, we present gap-dependent regret bounds for risk-sensitive RL. We first provide regret upper bounds for Algorithms \ref{alg:rsv2} and \ref{alg:rsq2}, 
and then we present a regret lower bound that any algorithm has to incur. For notational simplicity, we write $\Delta_{\mmin} \coloneqq \Delta_{\mmin, \beta}$ and $\Delta_{h} \coloneqq \Delta_{h, \beta}$ in short by dropping their dependency on $\beta$.

\subsection{Regret Upper Bounds}

The following theorem provides the gap-dependent performance of \cref{alg:rsv2}.

\begin{theorem}\label{thm:rsv2-bound-combined}
    For any fixed $\delta \in (0,1]$, with probability at least $1-\delta$, the regret of \cref{alg:rsv2} is upper bounded by  
    \begin{align*}
        \calR(K) \lesssim \frac{(e^{|\beta| H}-1)^2 H^3S^2A}{|\beta|^2\Delta_\mmin}\log(HSAK/\delta)^2.
    \end{align*}
    Moreover, the expected regret is  upper bounded by 
    \begin{align*}
        \expect[\calR(K)] \lesssim \frac{(e^{|\beta| H}-1)^2 H^3S^2A}{|\beta|^2\Delta_\mmin}\log(HSAK)^2.
    \end{align*}
\end{theorem}

The proof is provided in  \cref{sec:rsv2-bound}.
The above bounds are general as they hold for any $\beta \ne 0$. They also imply results obtained under the risk-neutral setting when $|\beta| \to 0$. This is verified given that $(e^{|\beta| H} - 1)/|\beta| \to H$ and $\Delta_\mmin \to \widetilde{\Delta}_\mmin$ (where we let $\widetilde{\Delta}_\mmin$ denote the minimal sub-optimality gap for the risk-neutral setting).
It can thus be seen that when $|\beta| \to 0$, \cref{thm:rsv2-bound-combined} provides a result that matches the risk-neutral bound 
$O((H^5d^3/\widetilde\Delta_\mmin)\log(HSAK/\delta)^2)$ (where $d=SA$ under our setting) 
in \citet[Theorem 4.4]{he2021logarithmic} with respect to $K$ and $H$.

Next we provide regret guarantees for \cref{alg:rsq2}.

\begin{theorem}\label{thm:rsq2-bound-combined}
    For any fixed $\delta \in (0,1]$, with probability at least $1-\delta$, the regret of \cref{alg:rsq2} is upper bounded by 
    \begin{align*}
        \calR(K) \lesssim \frac{(e^{|\beta| H}-1)^2 H^4 SA}{|\beta|^2\Delta_\mmin}\log(HSAK/\delta).
    \end{align*}
    Moreover, the expected regret is upper bounded by 
    \begin{align*}
        \expect[\calR(K)] \lesssim \frac{(e^{|\beta| H}-1)^2 H^4 SA}{|\beta|^2\Delta_\mmin}\log(HSAK).
    \end{align*}
\end{theorem}

The proof is provided in \cref{sec:rsq2-bound}. Note that the above regret bounds have the same factor $\frac{(e^{|\beta| H}-1)^2}{|\beta|^2\Delta_\mmin}$ as in \cref{thm:rsv2-bound-combined}; we will show in \cref{sec:reg_lower_bound} that such dependency is nearly optimal. Applying the same argument as for \cref{thm:rsv2-bound-combined}, when $|\beta| \to 0$, 
\cref{thm:rsq2-bound-combined} recovers the risk-neutral bound $O((H^6 SA/\widetilde\Delta_\mmin)\log(SAHK))$ proved in \citet[Theorem 3.1]{yang2021q} for a Q-learning algorithm. 

While the above discussion focuses on the case $|\beta| \to 0$, we also have the following result for $|\beta| \le 1/H$, which is more general.

\begin{corollary}\label{cor:risk_sens_bounds_coro}
    For any fixed $\delta \in (0,1]$, if $|\beta| \le \frac{1}{H}$, then with probability at least $1-\delta$ the regret of Algorithms \ref{alg:rsv2} and  \ref{alg:rsq2} is upper bounded by 
    \begin{align*}
       \calR(K) \lesssim  \begin{cases*}
                    \frac{H^5 S^{2}A}{\Delta_\mmin}\log(HSAK/\delta)^2, & for \cref{alg:rsv2};  \\
                     \frac{ H^6 SA}{\Delta_\mmin}\log(HSAK/\delta), & for \cref{alg:rsq2}.
                 \end{cases*}
    \end{align*}
    The expected regret of the two algorithms can be bounded similarly. 
\end{corollary}

\begin{proof}
The result follows from Theorems \ref{thm:rsv2-bound-combined} and \ref{thm:rsq2-bound-combined} by using the fact that the function $f(b) = \frac{e^{bx} - 1}{b}$ is increasing on $(0,\infty)$ for any $x > 0$ and $f(\frac{1}{x}) = (e-1)x \lesssim x$.
\end{proof}

\cref{cor:risk_sens_bounds_coro} states that as long as $|\beta|$ is sufficiently small, the regret of both algorithms can be bounded by quantities that are polynomial in $H$ (ignoring the possible $H$-dependence of $\Delta_\mmin$).



\paragraph{Comparison with existing works on risk-sensitive RL.}
Let us place  \cref{thm:rsv2-bound-combined,thm:rsq2-bound-combined} into the context of known results for risk-sensitive RL.
For ease of notation, we define the shorthand    $\poly(H,S,A;K,1/\delta)\defeq \poly(H,S,A)\cdot\polylog(K,1/\delta)$. 
Combining our results with existing regret bounds in \citet[Theorems 1 and 2]{fei2021exponential}, we have 
\ifDoubleColumn
    \begin{align}
        \begin{split}
            \calR(K) & \lesssim \frac{e^{|\beta| H}-1}{|\beta|} \cdot 
            \poly(H,S,A; K,1/\delta) \\
            & \qquad \cdot \min \Bigg\{\frac{e^{|\beta| H}-1}{|\beta|\Delta_\mmin}
            , K^{1/2} \Bigg\}.
        \end{split}
        \label{eq:combined_regret}
    \end{align}
\else
    \begin{align}
        \begin{split}
            \calR(K) \lesssim \frac{e^{|\beta| H}-1}{|\beta|} 
            \cdot \min \Bigg\{\frac{e^{|\beta| H}-1}{|\beta|\Delta_\mmin}
            , K^{1/2}
             \Bigg\}
             \cdot 
            \poly(H,S,A; K,1/\delta).
        \end{split} \label{eq:combined_regret}
    \end{align}
\fi
We see that the gap-dependent regret bounds in \cref{thm:rsv2-bound-combined,thm:rsq2-bound-combined} trade off the polynomial dependency on $K$ in the $\widetilde{O}(K^{1/2})$-regret (proved by \citet{fei2021exponential}) with a factor of $\frac{e^{|\beta| H}-1}{|\beta|\Delta_\mmin}$.
Since $\Delta_{\mmin} \in (0, \frac{1}{|\beta|}(e^{|\beta| H} - 1)]$, we may write $\Delta_{\mmin} = \frac{\mu}{|\beta|} (e^{|\beta| H} - 1) $ for some $\mu \in (0, 1]$. Then for $\mu \asymp 1$, the above regret bound becomes 
\begin{align*}
    \calR(K) \lesssim \frac{e^{|\beta| H}-1}{|\beta|} 
    \poly(H,S,A;K,1/\delta).
\end{align*}
Under this setting, we attain an \textit{exponential} improvement in $K$ over the existing regret bounds in \cite{fei2021exponential}, reducing the polynomial dependency on $K$ (specifically the  $\widetilde{O}(K^{1/2})$ dependency) to a logarithmic one. In sharp contrast, the regret bounds of \citet{fei2021exponential} that are independent of sub-optimality gaps, \ie, 
\ifDoubleColumn
\begin{align*}
    & \calR(K) \\ 
    & \quad \le \min   \Bigg\{HK, \frac{e^{|\beta| H}-1}{|\beta|}  K^{1/2}
    \poly(H,S,A; K,1/\delta)
     \Bigg\}.
\end{align*}
\else
\begin{align*}
    \calR(K) \le \min   \Bigg\{HK, \frac{e^{|\beta| H}-1}{|\beta|}  K^{1/2}
    \poly(H,S,A; K,1/\delta)
     \Bigg\}.
\end{align*}
\fi
must incur the exponential factor $\frac{e^{|\beta| H}-1}{|\beta|}$ for gaining only a polynomial improvement in $K$.
When $\mu \lesssim \frac{\log(K)}{\sqrt{K}}$, the regret bound \eqref{eq:combined_regret} is dominated by the existing $\widetilde{O}(K^{1/2})$  bound.


Our gap-dependent regret bounds also imply an exponential improvement in terms of sample complexity. Based on an argument in \citet{jin2018q}, our \cref{thm:rsv2-bound-combined,thm:rsq2-bound-combined} imply that \cref{alg:rsv2,alg:rsq2} find $\eps$-optimal policies in the PAC setting with $\widetilde \Omega\big(\frac{(e^{|\beta| H}-1)^2}{|\beta|^2 \Delta_\mmin \eps} \poly(H,S,A)\big)$ samples for any $\eps > 0$. On the other hand, the regret bounds in \citet{fei2021exponential} suggest sample complexity bounds on the order of  $\widetilde\Omega\big(\frac{(e^{|\beta| H}-1)^2}{|\beta|^2  \eps^2} \poly(H,S,A)\big)$.
Hence, when $\eps = \widetilde O(\frac{|\beta|\Delta_\mmin}{e^{|\beta| H}-1})$, our results translate to an exponential improvement in $|\beta|$ and $H$ in sample complexity bounds compared to those of \citet{fei2021exponential}.

\subsection{Regret Lower Bounds \label{sec:reg_lower_bound}}

Below we present regret lower bounds that complement the upper bounds in \cref{thm:rsv2-bound-combined,thm:rsq2-bound-combined}.

\begin{theorem}\label{thm:lower-combined}
    If $|\beta|(H-1) \geq \log 4$, $H \ge 2$, 
    $\Delta_\mmin \leq \frac{1}{8|\beta|}$, and $K \asymp \frac{1}{|\beta|^2\Delta_\mmin^2} (e^{|\beta|(H-1)}-1)$, then for any algorithm it holds that 
    \begin{align*}
        \expect[\calR(K)] \gtrsim \frac{e^{|\beta|(H-1)}-1}{|\beta|^2\Delta_\mmin};
    \end{align*}
    if $|\beta|(H-1)\leq \log H$, $H\geq 8$, 
    $\Delta_\mmin \leq \frac{1}{4|\beta|H}(e^{|\beta|(H-1)}-1)$, and $K \asymp \frac{1}{H|\beta|^2\Delta_\mmin^2} (e^{|\beta|(H-1)}-1)^2$, then for any algorithm it holds that 
    \begin{align*}
        \expect[\calR(K)] \gtrsim \frac{H}{\Delta_\mmin}.
    \end{align*}
\end{theorem}

We provide the proof in  \cref{sec:proof-lower}.
When $\beta$ is sufficiently large, \cref{thm:lower-combined} provides a lower bound with exponential dependence on $|\beta|$ and $H$, thus nearly matching
the upper bound in \cref{thm:rsv2-bound-combined} in terms of the exponential dependency 
and up to a logarithmic factor in $K$. Compared with the upper bound, the lower bound falls short of a term of  $e^{|\beta|(H-1)}-1$ as well as polynomial factors in other parameters; it is not yet clear whether there exists a fundamental gap between the two bounds, and we leave the investigation for future work. 

On the other hand, 
when $|\beta|$ is sufficiently small, 
we achieve a lower bound that depends only polynomially on $H$ and is independent of $\beta$ (beyond  potential dependence in  $\Delta_\mmin$). Consequently, this result nearly matches that of \cref{cor:risk_sens_bounds_coro}. 
Compared with existing risk-neutral lower bound of \citet[Theorem 5.4]{he2021logarithmic}, our result
specializes in the tabular setting and holds in the regime of non-vanishing $\beta$, while theirs adapts to linear function approximation but only in the risk-neutral regime ($|\beta| \to 0$).

To the best of our knowledge, this work presents the first non-asymptotic and gap-dependent regret bounds for risk-sensitive  \gls*{RL} based on the entropic risk measure.

\subsection{A Unified Framework \label{sec:unified}}
In existing literature, algorithms based on value iteration and $Q$-learning are often analyzed in independent ways due to their distinctive characteristics and update mechanism. We instead employ a unified  framework for analyzing the regret of \cref{alg:rsv2,alg:rsq2}. To that end, we focus on the high-probability regret bounds, and the expectation bound can be obtained as a by-product of the analysis. 
For each $k \in [K]$, let us define $\trj^k \defeq \{(s^k_h, a^k_h)\}_{h \in [H]}$ to be the  sample trajectory in episode $k$. 
Thanks to \cref{fact:risk-sen-gaps-meet-cond} that the cascaded gaps $\{\Delta_{h}\}$ satisfy \cref{cond:bellman_diff}, we may derive the following lemma on regret using a standard concentration result.

\begin{lemma}\label{lem:high-prob-bound}
    For Algorithms \ref{alg:rsv2} and \ref{alg:rsq2} and any fixed $\delta \in (0,1]$, 
    it holds with probability at least $1-\delta/2$ that 
    \ifDoubleColumn
        \begin{align*}
            \calR(K) & \lesssim \sumk\sumh \Dh(\skh,\akh;\Tkh) \\
            & \quad + \frac{e^{|\beta| H}-1}{|\beta|} H\log(\log K/\delta).
        \end{align*}
    \else
        \begin{align*}
            \calR(K) & \lesssim \sumk\sumh \Dh(\skh,\akh;\Tkh) + \frac{e^{|\beta| H}-1}{|\beta|} H\log(\log K/\delta).
        \end{align*}
    \fi
\end{lemma}

The proof is given in \cref{sec:proof_high_prob_bound}. In the above lemma, the regret $\calR(K)$ plays a role similar to the expectation of the random variable $\sum_{k,h} \Dh(\skh,\akh;\Tkh)$, while the second term on the RHS 
can be interpreted as the deviation of the random variable from its expectation. With \cref{lem:high-prob-bound} in place, it remains to bound the first term of RHS for both 
algorithms.
We do so in the next two lemmas.

\begin{lemma}\label{lem:rsv2-gap-bound}
    For \cref{alg:rsv2} and any $\delta \in (0,1]$, it holds with probability at least $1-\delta/2$ that 
    \ifDoubleColumn
        \begin{align*}
            & \sumk\sumh \Dh(\skh,\akh;\Tkh) \\
            & \qquad \lesssim \frac{(e^{|\beta| H}-1)^2 H^3S^2A}{|\beta|^2\Delta_\mmin}\log(2HSAK/\delta)^2.
        \end{align*}
    \else
        \begin{align*}
            \sumk\sumh \Dh(\skh,\akh;\Tkh) \lesssim \frac{(e^{|\beta| H}-1)^2 H^3S^2A}{|\beta|^2\Delta_\mmin}\log(2HSAK/\delta)^2.
        \end{align*}
    \fi
\end{lemma}

\begin{lemma}\label{lem:rsq2-gap-bound}
    For \cref{alg:rsq2} and any $\delta \in (0,1]$, it holds with probability at least $1-\delta/2$ that
    \ifDoubleColumn
        \begin{align*}
            & \sumk\sumh \Dh(\skh,\akh;\Tkh) \\
            & \qquad \lesssim \frac{(e^{|\beta| H}-1)^2 H^4 SA}{|\beta|^2\Delta_\mmin}\log(2HSAK/\delta).
        \end{align*}
    \else
        \begin{align*}
            \sumk\sumh \Dh(\skh,\akh;\Tkh) \lesssim \frac{(e^{|\beta| H}-1)^2 H^4 SA}{|\beta|^2\Delta_\mmin}\log(2HSAK/\delta).
        \end{align*}
    \fi
\end{lemma}
We provide the proofs in Appendices \ref{sec:proof-rsv2-gap-bound} and \ref{sec:proof-rsq2-gap-bound}. By combining \cref{lem:high-prob-bound} with \cref{lem:rsv2-gap-bound,lem:rsq2-gap-bound}, we arrive at \cref{thm:rsv2-bound-combined,thm:rsq2-bound-combined}, respectively.
In addition, we remark that the bounds on expected regret can also be obtained from  \cref{lem:rsv2-gap-bound,lem:rsq2-gap-bound} for corresponding algorithms by simple calculations.

\section{Conclusion}

We study gap-dependent regret for risk-sensitive \gls*{RL}  with the entropic risk measure under episodic and finite-horizon \glspl*{MDP}. We propose a novel definition of sub-optimality gaps, named as cascaded gaps, tailored to the unique characteristics of risk-sensitive \gls*{RL}. We prove gap-dependent lower bounds on the regret to be incurred by any algorithm, and provide nearly matching upper bounds for two existing model-free algorithms. Under proper settings, we demonstrate that  our upper bounds imply exponential improvement in bounds of both regret and sample complexity over existing results.

\clearpage
\bibliography{risk-sensitive-ref}

\begin{thebibliography}{}

\bibitem[B{\"a}uerle \& Rieder, 2014]{bauerle2014more}
B{\"a}uerle, N. \& Rieder, U. (2014).
\newblock More risk-sensitive {M}arkov decision processes.
\newblock {\em Mathematics of Operations Research}, 39(1), 105--120.

\bibitem[Borkar, 2001]{borkar2001sensitivity}
Borkar, V.~S. (2001).
\newblock A sensitivity formula for risk-sensitive cost and the actor-critic
  algorithm.
\newblock {\em Systems \& Control Letters}, 44(5), 339--346.

\bibitem[Borkar, 2002]{borkar2002q}
Borkar, V.~S. (2002).
\newblock Q-learning for risk-sensitive control.
\newblock {\em Mathematics of Operations Research}, 27(2), 294--311.

\bibitem[Borkar \& Meyn, 2002]{borkar2002risk}
Borkar, V.~S. \& Meyn, S.~P. (2002).
\newblock Risk-sensitive optimal control for {M}arkov decision processes with
  monotone cost.
\newblock {\em Mathematics of Operations Research}, 27(1), 192--209.

\bibitem[Cavazos-Cadena \& Fern{\'a}ndez-Gaucherand,
  2000]{cavazos2000vanishing}
Cavazos-Cadena, R. \& Fern{\'a}ndez-Gaucherand, E. (2000).
\newblock The vanishing discount approach in {M}arkov chains with
  risk-sensitive criteria.
\newblock {\em IEEE Transactions on Automatic Control}, 45(10), 1800--1816.

\bibitem[Cesa-Bianchi \& Lugosi, 2006]{cesa2006prediction}
Cesa-Bianchi, N. \& Lugosi, G. (2006).
\newblock {\em Prediction, learning, and games}.
\newblock Cambridge university press.

\bibitem[Chen et~al., 2021a]{chen2020multiple}
Chen, L., Min, Y., Belkin, M., \& Karbasi, A. (2021a).
\newblock Multiple descent: Design your own generalization curve.
\newblock In {\em NeurIPS}.

\bibitem[Chen et~al., 2020]{chen2020more}
Chen, L., Min, Y., Zhang, M., \& Karbasi, A. (2020).
\newblock More data can expand the generalization gap between adversarially
  robust and standard models.
\newblock In {\em International Conference on Machine Learning}  (pp.\
  1670--1680).: PMLR.

\bibitem[Chen et~al., 2021b]{chen2021infinite}
Chen, L., Scherrer, B., \& Bartlett, P.~L. (2021b).
\newblock Infinite-horizon offline reinforcement learning with linear function
  approximation: Curse of dimensionality and algorithm.
\newblock {\em arXiv preprint arXiv:2103.09847}.

\bibitem[Chen \& Xu, 2021]{chen2020deep}
Chen, L. \& Xu, S. (2021).
\newblock Deep neural tangent kernel and laplace kernel have the same rkhs.
\newblock In {\em ICLR}.

\bibitem[Coraluppi \& Marcus, 1999]{coraluppi1999risk}
Coraluppi, S.~P. \& Marcus, S.~I. (1999).
\newblock Risk-sensitive and minimax control of discrete-time, finite-state
  {M}arkov decision processes.
\newblock {\em Automatica}, 35(2), 301--309.

\bibitem[Di~Masi \& Stettner, 1999]{di1999risk}
Di~Masi, G.~B. \& Stettner, L. (1999).
\newblock Risk-sensitive control of discrete-time {M}arkov processes with
  infinite horizon.
\newblock {\em SIAM Journal on Control and Optimization}, 38(1), 61--78.

\bibitem[Eriksson \& Dimitrakakis, 2019]{eriksson2019epistemic}
Eriksson, H. \& Dimitrakakis, C. (2019).
\newblock Epistemic risk-sensitive reinforcement learning.
\newblock {\em arXiv preprint arXiv:1906.06273}.

\bibitem[Fei \& Chen, 2018a]{fei2018exponential}
Fei, Y. \& Chen, Y. (2018a).
\newblock Exponential error rates of sdp for block models: Beyond
  grothendieck’s inequality.
\newblock {\em IEEE Transactions on Information Theory}, 65(1), 551--571.

\bibitem[Fei \& Chen, 2018b]{fei2018hidden}
Fei, Y. \& Chen, Y. (2018b).
\newblock Hidden integrality of sdp relaxations for sub-gaussian mixture
  models.
\newblock In {\em Conference On Learning Theory}  (pp.\ 1931--1965).: PMLR.

\bibitem[Fei \& Chen, 2020]{fei2020achieving}
Fei, Y. \& Chen, Y. (2020).
\newblock Achieving the bayes error rate in synchronization and block models by
  sdp, robustly.
\newblock {\em IEEE Transactions on Information Theory}, 66(6), 3929--3953.

\bibitem[Fei et~al., 2021a]{fei2021exponential}
Fei, Y., Yang, Z., Chen, Y., \& Wang, Z. (2021a).
\newblock Exponential bellman equation and improved regret bounds for
  risk-sensitive reinforcement learning.
\newblock {\em arXiv preprint arXiv:2111.03947}.

\bibitem[Fei et~al., 2020]{fei2020risk}
Fei, Y., Yang, Z., Chen, Y., Wang, Z., \& Xie, Q. (2020).
\newblock Risk-sensitive reinforcement learning: Near-optimal risk-sample
  tradeoff in regret.
\newblock In {\em Advances in Neural Information Processing Systems}.

\bibitem[Fei et~al., 2021b]{fei2021risk}
Fei, Y., Yang, Z., \& Wang, Z. (2021b).
\newblock Risk-sensitive reinforcement learning with function approximation: A
  debiasing approach.
\newblock In {\em International Conference on Machine Learning}  (pp.\
  3198--3207).: PMLR.

\bibitem[Fleming \& McEneaney, 1995]{fleming1995risk}
Fleming, W.~H. \& McEneaney, W.~M. (1995).
\newblock Risk-sensitive control on an infinite time horizon.
\newblock {\em SIAM Journal on Control and Optimization}, 33(6), 1881--1915.

\bibitem[F{\"o}llmer \& Knispel, 2011]{follmer2011entropic}
F{\"o}llmer, H. \& Knispel, T. (2011).
\newblock Entropic risk measures: Coherence vs. convexity, model ambiguity and
  robust large deviations.
\newblock {\em Stochastics and Dynamics}, 11(02n03), 333--351.

\bibitem[Hansen \& Sargent, 2011]{hansen2011robustness}
Hansen, L.~P. \& Sargent, T.~J. (2011).
\newblock {\em Robustness}.
\newblock Princeton university press.

\bibitem[He et~al., 2021]{he2021logarithmic}
He, J., Zhou, D., \& Gu, Q. (2021).
\newblock Logarithmic regret for reinforcement learning with linear function
  approximation.
\newblock In {\em International Conference on Machine Learning}  (pp.\
  4171--4180).: PMLR.

\bibitem[Hern{\'a}ndez-Hern{\'a}ndez \& Marcus, 1996]{hernandez1996risk}
Hern{\'a}ndez-Hern{\'a}ndez, D. \& Marcus, S.~I. (1996).
\newblock Risk sensitive control of {M}arkov processes in countable state
  space.
\newblock {\em Systems \& Control Letters}, 29(3), 147--155.

\bibitem[Howard \& Matheson, 1972]{howard1972risk}
Howard, R.~A. \& Matheson, J.~E. (1972).
\newblock Risk-sensitive {M}arkov decision processes.
\newblock {\em Management Science}, 18(7), 356--369.

\bibitem[Jacobson, 1973]{jacobson1973optimal}
Jacobson, D. (1973).
\newblock Optimal stochastic linear systems with exponential performance
  criteria and their relation to deterministic differential games.
\newblock {\em IEEE Transactions on Automatic control}, 18(2), 124--131.

\bibitem[Jin et~al., 2018]{jin2018q}
Jin, C., Allen-Zhu, Z., Bubeck, S., \& Jordan, M.~I. (2018).
\newblock Is {Q}-learning provably efficient?
\newblock In {\em Advances in Neural Information Processing Systems}  (pp.\
  4863--4873).

\bibitem[Lattimore \& Szepesv{\'a}ri, 2020]{lattimore2020bandit}
Lattimore, T. \& Szepesv{\'a}ri, C. (2020).
\newblock {\em Bandit algorithms}.
\newblock Cambridge University Press.

\bibitem[Ling et~al., 2019]{ling2019landscape}
Ling, S., Xu, R., \& Bandeira, A.~S. (2019).
\newblock On the landscape of synchronization networks: A perspective from
  nonconvex optimization.
\newblock {\em SIAM Journal on Optimization}, 29(3), 1879--1907.

\bibitem[Mihatsch \& Neuneier, 2002]{mihatsch2002risk}
Mihatsch, O. \& Neuneier, R. (2002).
\newblock Risk-sensitive reinforcement learning.
\newblock {\em Machine Learning}, 49(2-3), 267--290.

\bibitem[Min et~al., 2021a]{min2021curious}
Min, Y., Chen, L., \& Karbasi, A. (2021a).
\newblock The curious case of adversarially robust models: More data can help,
  double descend, or hurt generalization.
\newblock In {\em Uncertainty in Artificial Intelligence}  (pp.\ 129--139).:
  PMLR.

\bibitem[Min et~al., 2021b]{min2021learning}
Min, Y., He, J., Wang, T., \& Gu, Q. (2021b).
\newblock Learning stochastic shortest path with linear function approximation.
\newblock {\em arXiv preprint arXiv:2110.12727}.

\bibitem[Min et~al., 2021c]{min2021variance}
Min, Y., Wang, T., Zhou, D., \& Gu, Q. (2021c).
\newblock Variance-aware off-policy evaluation with linear function
  approximation.
\newblock In {\em Advances in Neural Information Processing Systems}.

\bibitem[Nass et~al., 2019]{nass2019entropic}
Nass, D., Belousov, B., \& Peters, J. (2019).
\newblock Entropic risk measure in policy search.
\newblock In {\em 2019 IEEE/RSJ International Conference on Intelligent Robots
  and Systems (IROS)}  (pp.\ 1101--1106).: IEEE.

\bibitem[Niv et~al., 2012]{niv2012neural}
Niv, Y., Edlund, J.~A., Dayan, P., \& O'Doherty, J.~P. (2012).
\newblock Neural prediction errors reveal a risk-sensitive
  reinforcement-learning process in the human brain.
\newblock {\em Journal of Neuroscience}, 32(2), 551--562.

\bibitem[Ortega \& Braun, 2013]{ortega2013thermodynamics}
Ortega, P.~A. \& Braun, D.~A. (2013).
\newblock Thermodynamics as a theory of decision-making with
  information-processing costs.
\newblock {\em Proceedings of the Royal Society A: Mathematical, Physical and
  Engineering Sciences}, 469(2153), 20120683.

\bibitem[Ortega \& Stocker, 2016]{ortega2016human}
Ortega, P.~A. \& Stocker, A.~A. (2016).
\newblock Human decision-making under limited time.
\newblock {\em arXiv preprint arXiv:1610.01698}.

\bibitem[Osogami, 2012]{osogami2012robustness}
Osogami, T. (2012).
\newblock Robustness and risk-sensitivity in {M}arkov decision processes.
\newblock In {\em Advances in Neural Information Processing Systems}  (pp.\
  233--241).

\bibitem[Shen et~al., 2013]{shen2013risk}
Shen, Y., Stannat, W., \& Obermayer, K. (2013).
\newblock Risk-sensitive {M}arkov control processes.
\newblock {\em SIAM Journal on Control and Optimization}, 51(5), 3652--3672.

\bibitem[Shen et~al., 2014]{shen2014risk}
Shen, Y., Tobia, M.~J., Sommer, T., \& Obermayer, K. (2014).
\newblock Risk-sensitive reinforcement learning.
\newblock {\em Neural Computation}, 26(7), 1298--1328.

\bibitem[Simchowitz \& Jamieson, 2019]{simchowitz2019non}
Simchowitz, M. \& Jamieson, K.~G. (2019).
\newblock Non-asymptotic gap-dependent regret bounds for tabular mdps.
\newblock {\em Advances in Neural Information Processing Systems}, 32,
  1153--1162.

\bibitem[Simon, 1955]{simon1955behavioral}
Simon, H.~A. (1955).
\newblock A behavioral model of rational choice.
\newblock {\em The quarterly journal of economics}, 69(1), 99--118.

\bibitem[Song et~al., 2021]{song2021convergence}
Song, G., Xu, R., \& Lafferty, J. (2021).
\newblock Convergence and alignment of gradient descent with random back
  propagation weights.
\newblock {\em arXiv preprint arXiv:2106.06044}.

\bibitem[Whittle, 1990]{whittle1990risk}
Whittle, P. (1990).
\newblock {\em Risk-sensitive {O}ptimal {C}ontrol}, volume~20.
\newblock Wiley New York.

\bibitem[Williams et~al., 2016]{williams2016aggressive}
Williams, G., Drews, P., Goldfain, B., Rehg, J.~M., \& Theodorou, E.~A. (2016).
\newblock Aggressive driving with model predictive path integral control.
\newblock In {\em 2016 IEEE International Conference on Robotics and Automation
  (ICRA)}  (pp.\ 1433--1440).: IEEE.

\bibitem[Williams et~al., 2017]{williams2017information}
Williams, G., Wagener, N., Goldfain, B., Drews, P., Rehg, J.~M., Boots, B., \&
  Theodorou, E.~A. (2017).
\newblock Information theoretic mpc for model-based reinforcement learning.
\newblock In {\em 2017 IEEE International Conference on Robotics and Automation
  (ICRA)}  (pp.\ 1714--1721).: IEEE.

\bibitem[Xu et~al., 2021]{xu2021meta}
Xu, R., Chen, L., \& Karbasi, A. (2021).
\newblock Meta learning in the continuous time limit.
\newblock In {\em International Conference on Artificial Intelligence and
  Statistics}  (pp.\ 3052--3060).: PMLR.

\bibitem[Yang et~al., 2021]{yang2021q}
Yang, K., Yang, L., \& Du, S. (2021).
\newblock Q-learning with logarithmic regret.
\newblock In {\em International Conference on Artificial Intelligence and
  Statistics}  (pp.\ 1576--1584).: PMLR.

\end{thebibliography}
\bibliographystyle{apalike2}

\clearpage
\ifDoubleColumn
    \onecolumn
\fi
\appendix

\section{Additional Definitions}

Before diving into the proofs, we would like to provide additional definitions on several notion of gaps.
We start with the definition of policy-controlled trajectories and sample trajectories as series of state-action pairs; we then define several notions of semi-normalized gaps that we use only in the proofs.

For a policy $\pi$, we define the  $\pi$-controlled trajectory $\trj^\pi \defeq \{(s_j,\pi_j(s_j))\}_{j\in[H]}$ as a series of state-action pairs where the action follows $\pi$ at every state. We define $\trj^k$ be the sample trajectory of episode $k$, \ie, $\trj^k \defeq \{(s^k_j, a^k_j)\}_{j\in[H]}$.
Let us introduce some additional notion of gaps, based upon cascaded gaps, to assist our proofs. Without loss of generality, we fix $\beta \ne 0$, a trajectory $\trj$, and $(k,h,s,a)\in [K]\times[H]\times\calS\times\calA$.
We define the semi-normalized sub-optimality gap as
\begin{align*}
    \bDh(s,a;\tau_{h-1}) \defeq \frac{1}{\beta} e^{\beta\cdot\sum_{j=1}^{h-1} r_j(s_j,a_j)} \left[ e^{\beta\cdot\Vsh(s)} - e^{\beta\cdot\Qsh(s,a)} \right],
\end{align*}
and we also pair the semi-normalized gap with a semi-normalizer 
\begin{align*}
    \bar\psi_\beta \defeq
    \begin{cases}
        1, & \beta>0; \\
        e^{-\beta H}, & \beta<0.
    \end{cases}
\end{align*}
Note that the cascaded gap $\Dh$ satisfies that $\Dh(s,a;\tau_{h-1}) = \bar\psi_\beta \cdot \bDh(s,a;\tau_{h-1})$, which can be regarded as a further level of normalization.
For any policy $\pi$, we also define the $\pi$-controlled sub-optimality gap as
\begin{align*}
    \bDh^{\pi}(s,a;\trj_{h-1}) \defeq \frac{1}{\beta} e^{\beta\cdot\sum_{j=1}^{h-1} r_j(s_j,a_j)} \left[ e^{\beta\cdot\Vsh(s)} - e^{\beta\cdot\Qpih(s,a)} \right],
\end{align*} 
which characterizes the sub-optimality of policy $\pi$ with respect to the optimal policy $\pi^*$.
Similar to the semi-normalized sub-optimality gap, we define the \emph{normalized} $\pi$-controlled sub-optimality gap  to be $\Dh^{\pi}(s,a;\tau_{h-1}) \defeq \bar\psi_\beta \cdot \bDh^{\pi}(s,a;\tau_{h-1})$, where the semi-normalizer is applied.
Notice that $\Vsh(s) \geq \Qsh(s,a) \geq \Qpih(s,a)$ for any $(s,a)\in\calS\times\calA$ by definition, and the gaps are always non-negative quantities due to the monotonicity of exponential function and the normalization factor $\frac{1}{\beta}$. The semi-normalizer $\bar\psi_\beta$ is designed to keep the gaps on the same magnitude for both $\beta>0$ and $\beta<0$.  

We introduce a notion of optimism gap that represents the difference between the optimistic estimation $\Qkh$ by the algorithm and the optimal value function $\Vsh$.
Similar to the cascaded gap, we define the semi-normalized optimism gap as
\begin{align*}
    \bDkh(s,a;\tau_{h-1}) \defeq \frac{1}{\beta} e^{\beta\cdot\sum_{j=1}^{h-1} r_j(s_j,a_j)} \left[ e^{\beta\cdot \Qkh(s,a)} - e^{\beta\cdot \Vsh(s)} \right],
\end{align*}
and the \emph{normalized} optimism gap as $\Dkh(s,a;\tau_{h-1}) \defeq \bar\psi_\beta \cdot \bDkh(s,a;\tau_{h-1})$, with the same semi-normalizer applied.

Moreover, we define the (semi-normalized) minimal sub-optimality gap to be the minimal non-zero semi-normalized sub-optimality gap over the tuple $(h,s,a,\trj)$: 
\begin{align*}
    \xoverline\Delta_\mmin \defeq \min_{h,s,a,\trj} \{\bDh(s,a;\tau_{h-1}) \st \bDh(s,a;\tau_{h-1})\neq 0\}.
\end{align*}
Note that the dependency on $\beta$ is implicit here.
With the above definition, we recall the minimal sub-optimality gap from \cref{eq:gap_min}, and have that 
$\Delta_\mmin = \bar\psi_\beta\xoverline\Delta_\mmin$. 

In the subsequent proofs we will  leverage a peeling argument, for which 
we define a series of end points $\{\rho_n\}_{n=1}^N$, where $\rho_n \defeq 2^n\Delta_\mmin$, and they generate a series of intervals $\{I_n\}_{n=1}^N$ with $I_n \defeq [\rho_{n-1},\rho_n)$ for all $n\in[N]$.

Recall that $s^k_1$ is defined as the state in the first step of episode $k$; since we assume fixed initial state $s_1$ for all episodes, we have $s_1^k = s_1$. We introduce the notion of exponential regret that sums over all the episodes the difference between exponential value functions of the optimal policy $\pi^*$ and that of any policy $\pi^k$.
Specifically, for any episodic \gls*{MDP} with $K$ episodes, the exponential regret of policy $\{\pi^k\}_{k=1}^K$ is defined as $\calE(K) \defeq \frac{1}{\beta} \sumk [e^{\beta\cdot V_1^*} - e^{\beta\cdot V_1^{\pi^k}}](s_1^k)$.

\section{Proofs of Upper Bounds}

\subsection{Proof of \cref{lem:high-prob-bound} \label{sec:proof_high_prob_bound}}

In this proof, we assume $\beta>0$ without loss of generality, the proof where $\beta<0$ can be similarly carried out. Let us denote $Z_k \defeq \sumh\bDh(\skh,\akh;\Tkh) - \frac{1}{\beta} [e^{\beta\cdot V_1^*} - e^{\beta\cdot V_1^{\pi^k}}](s_1^k)$. Following from \cref{lem:exp-regret-decomposition}, we have $\{Z_k\}_{k\in[K]}$ being a martingale difference sequence with respect to the filtration $\calF_k$ that represents all the randomness up to episode $k$. Further recall that the semi-normalized sub-optimality gap for any trajectory
\begin{align*}
    \bDh(s_h,a_h;\tau_{h-1}) = \frac{1}{\beta}e^{\beta\cdot\sum_{j=1}^{h-1} r_j(s_j,a_j)} \left[ e^{\beta\cdot\Vsh(s_h)} - e^{\beta\cdot\Qsh(s_h,a_h)} \right] \geq 0,
\end{align*}
and we can thus control the magnitude of $Z_k$ by
\begin{align*}
    |Z_k| \leq \sumh |\bDh(\skh,\akh;\Tkh)| \leq \frac{H}{|\beta|}|e^{\beta H}-1| \eqdef B_\beta.
\end{align*}
For any trajectory $\{\Tkh\}_{h,k}$, if the exponential regret $\calE(K) = \frac{1}{\beta} \sumk [e^{\beta\cdot V_1^*} - e^{\beta\cdot V_1^{\pi^k}}](s_1^k) \leq B_\beta$, then the sum of $Z_k$ can be lower bounded through applying the definition of $Z_k$:
\begin{align*}
    \sumk Z_k \geq -\frac{1}{\beta} \sumk (e^{\beta\cdot V_1^*} - e^{\beta\cdot V_1^{\pi^k}})(s_1^k) \geq - B_\beta.
\end{align*}
    
Otherwise, if $\calE(K) > B_\beta$, we lower bound the sum $\sumk Z_k$ following Freedman inequality from \cref{lem:freedman}. 
More specifically, notice that given the filtration $\calF_k$, the variance $\chi = \sumk \expect[Z_k^2\given \calF_k]$ over all $Z_k$'s is upper bounded by
\begin{align*}
    \chi & \labelrel\leq{ineqn:variance-bound} \sumk \expect[(Z_k + \frac{1}{\beta}(e^{\beta\cdot V_1^*} - e^{\beta\cdot V_1^{\pi^k}})(s_1^k))^2 \given \calF_k] \\
    & = \sumk \expect\Big[ \Big( \sumh\bDh(\skh,\akh;\Tkh) \Big)^2\given \calF_k \Big] \\
    & \labelrel\leq{ineqn:bounded-gap} \sumk B_\beta \cdot \expect\Big[ \sumh\bDh(\skh,\akh;\Tkh)\given \calF_k \Big] \\
    & = B_\beta \sumk \frac{1}{\beta} (e^{\beta\cdot V_1^*} - e^{\beta\cdot V_1^{\pi^k}})(s_1^k)  \\
    & = B_\beta \cdot \calE(K),
\end{align*}
where step \eqref{ineqn:variance-bound} follows from $\expect[(X-\expect X)^2] \leq \expect X^2$ for any random variable $X$, and step \eqref{ineqn:bounded-gap} follows from the fact that $\sumh\bDh(\skh,\akh;\Tkh)\leq B_\beta$. 
For any $\varsigma>0$ and $\varrho\in\ZZ_+$, we let
\begin{align*}
    v_i \defeq \frac{2^i}{K}B_\beta \cdot \calE(K) = 2^i \frac{H}{|\beta|^2}(e^{\beta H}-1)^2
\end{align*}
and
\begin{align*}
    u_i \defeq \sqrt{2^{i+1}\frac{H}{|\beta|^2}(e^{\beta H}-1)^2\varsigma} + \frac{2H|e^{\beta H}-1|\varsigma}{3|\beta|}
\end{align*}
for each $i\in[\varrho]$, and the corresponding concentration inequality $\prob [ \sumk Z_k \leq -u_i, \ \chi \leq v_i] \leq e^{-\varsigma}$ follows from \cref{lem:freedman}.  
Let us denote the shorthand $U \defeq 2\sqrt{\calE(K)\frac{H}{|\beta|} |e^{\beta H}-1|\varsigma} + \frac{2H|e^{\beta H}-1|\varsigma}{3|\beta|}$ and $\underline B_\beta \defeq \frac{1}{|\beta|} |e^{\beta H}-1| \leq B_\beta$, then it holds that
\begin{align*}
    \prob \Big[ \sumk Z_k \leq -U,\ \ \calE(K) > B_\beta \Big] \leq & \prob \Big[ \sumk Z_k \leq -U,\ \  \calE(K) > \underline B_\beta \Big],
\end{align*}
and we can bound the RHS following a peeling argument. Notice that event $\calG \subseteq \bigcup_{i=1}^\varrho \calG_i$, where we denote the events $\calG \defeq \{\frac{1}{|\beta|} |e^{\beta H}-1| < \calE(K) \leq \frac{K}{|\beta|} |e^{\beta H}-1|\}$ and $\calG_i \defeq \{\frac{2^{i-1}}{|\beta|}|e^{\beta H}-1| < \calE(K) \leq \frac{2^i}{|\beta|}|e^{\beta H}-1|\}$ for all $i\in[\varrho]$. It follows the definition that
\begin{align*}
    \prob \Big[ \sumk Z_k \leq -U,\ \ \calE(K) > \underline B_\beta \Big]
    & \labelrel={eqn:exponential-regret} \prob \Big[ \sumk Z_k \leq -U,\ \ \chi \leq B_\beta \cdot \calE(K),\ \ \calG \Big] \\
    & \labelrel\leq{ineqn:stratify} \sum_{i=1}^\varrho \prob \Big[ \sumk Z_k \leq -U,\ \ \chi \leq B_\beta \cdot \calE(K),\ \  \calG_i \Big] \\
    & \labelrel\leq{ineqn:relax-expo-regret} \sum_{i=1}^\varrho \prob \Big[ \sumk Z_k \leq -u_i,\ \ \chi \leq v_i \Big] \\
    & \leq \varrho  e^{-\varsigma},
\end{align*}
where step \eqref{eqn:exponential-regret} follows from the fact that $\calE(K) \leq K|e^{\beta H}-1|/|\beta|$, step \eqref{ineqn:stratify} follows from stratifying the feasible range into $\varrho=\lceil\log K\rceil$ layers and applying union bound over all $i\in[\varrho]$, and step \eqref{ineqn:relax-expo-regret} follows from relaxing the quantity of $\calE(K)$ within the stratified range $2^{i-1}|e^{\beta H}-1|/|\beta| < \calE(K) \leq 2^i|e^{\beta H}-1|/|\beta|$ for each $i\in[\varrho]$. 
Combining both cases, with probability at least $1-\varrho \cdot e^{-\varsigma}$, we have a lower bound
\begin{align*}
    \sumk Z_k & \geq \min\{-U, -B_\beta\} \geq -U - B_\beta.
\end{align*}
Recall that $\sumk Z_k = \sumk\sumh \bDh(\skh,\akh;\Tkh) - \calE(K)$. The lower bound on $\sumk Z_k$ implies an upper bound on exponential regret:
\begin{align*}
    \calE(K) \leq 2\sqrt{\calE(K) \cdot \frac{H}{|\beta|} |e^{\beta H}-1|\varsigma} + \sumk\sumh\bDh(\skh,\akh;\Tkh) + \frac{2H|e^{\beta H}-1|\varsigma}{3|\beta|} + \frac{H}{|\beta|} |e^{\beta H}-1|,
    \end{align*}
and a sufficient condition gives that with probability at least $1-\delta/2$
\begin{align*}
    \calE(K) \leq 2\sumk\sumh \bDh(\skh,\akh;\Tkh) + \frac{16H|e^{\beta H}-1|\varsigma}{3|\beta|} + \frac{2H}{|\beta|} |e^{\beta H}-1|,
\end{align*}
where $\varsigma = \log(2\lceil\log K\rceil/\delta)$.
Following \cref{lem:regret-bound}, with probability at least $1-\delta/2$, the total regret $\calR(K)$ is bounded by
\begin{align*}
    \calR(K) 
    &\leq \bar\psi_\beta \cdot\calE(K) \\
    & \leq 2\bar\psi_\beta \sumk\sumh \bDh(\skh,\akh;\Tkh) + \frac{(e^{|\beta| H}-1)(16H\log(2\lceil\log K\rceil/\delta)+6H)}{3|\beta|} \\
    & = 2\sumk\sumh \Dh(\skh,\akh;\Tkh) + \frac{(e^{|\beta| H}-1)(16H\log(2\lceil\log K\rceil/\delta)+6H)}{3|\beta|}.
\end{align*}

\begin{lemma}\label{lem:exp-regret-decomposition}
    The quantity $\frac{1}{\beta}[e^{\beta\cdot V_1^*} - e^{\beta\cdot V_1^{\pi^k}}](s_1^k)$ for episode $k \in [K]$ admits the following decomposition:
    \begin{align*}
        \frac{1}{\beta}[e^{\beta\cdot V_1^*} - e^{\beta\cdot V_1^{\pi^k}}](s_1^k) = \expect\Big[\sumh \bDh(s_h,\pi_h^k(s_h);\tau^{\pi^k}_{h-1}) \given \calF_k\Big].
    \end{align*}
\end{lemma}

\begin{proof}
    For any episode $k\in[K]$, we have
    \begin{align*}
        \frac{1}{\beta}[e^{\beta\cdot V_1^*} - e^{\beta\cdot V_1^{\pi^k}}](s_1^k) 
        = \xoverline\Delta_1(\tau_1^{\pi^k}) + \frac{1}{\beta} e^{\beta\cdot r_1(s_1,\pi_1^k(s_1))}\expect_{s_2}[(e^{\beta\cdot V_2^*} - e^{\beta\cdot V_2^{\pi^k}})(s_2)\given \calF_k], 
    \end{align*}
    where 
    the equality is due to \cref{fact:risk-sen-gaps-meet-cond} and 
    the expectation is taken over the transition probability $\calP_1(\cdot\given s_1, \pi_1^k(s_1))$ given the policy $\pi^k$. 
    We expand the RHS of the equation recursively to get
    \begin{align*}
        \frac{1}{\beta}[e^{\beta\cdot V_1^*} - e^{\beta\cdot V_1^{\pi^k}}](s_1^k) = \sumh\expect\Big[\bDh(s_h,\pi_h^k(s_h);\tau^{\pi^k}_{h-1}) \given \calF_k\Big],
    \end{align*}
    where the sub-optimality gap is defined over the entire trajectory $\tau_{h-1}^{\pi^k}$ and the expectation is over all trajectories reachable under the policy $\pi^k$ and transition probability $\calP_h(\cdot\given s_h, \pi_h^k(s_h))$ for all $h\in[H]$.
\end{proof}

In the sections below, we provide for each algorithm a near-optimal upper bound of the sum of cascaded gaps $\sumk\sumh \Dh(\skh,\akh;\Tkh)$ following a peeling argument, a widely used technique for empirical processes \citep{yang2021q,he2021logarithmic}. 
The final results, \ie, \cref{thm:rsv2-bound-combined,thm:rsq2-bound-combined}, follow from plugging in the upper bounds of the sum of cascaded gaps into \cref{lem:high-prob-bound}.

\subsection{Upper Bounds for \cref{alg:rsv2}}\label{sec:rsv2-bound}

\subsubsection{Proof of \cref{thm:rsv2-bound-combined}}




\paragraph{High-probability regret bound.}

    Following \cref{lem:high-prob-bound,lem:rsv2-gap-bound}, with probability at least $1-\delta$ we have 
    \begin{align*}
        \calR(K) & \lesssim \sumk\sumh \Dh(\skh,\akh;\Tkh) + \frac{(e^{|\beta| H}-1)H\log(\log K/\delta)}{|\beta|} \\
        & \lesssim \frac{(e^{|\beta| H}-1)^2 H^3S^2A\log(2HSAK/\delta)^2}{|\beta|^2\Delta_\mmin},
    \end{align*}
    where the last inequality is due to $\Delta_\mmin \leq \frac{1}{|\beta|} (e^{|\beta| H}-1)$.


\paragraph{Expected regret bound.}
    Recall from \cref{lem:regret-bound} that $\calR(K) \leq \bar\psi_\beta \calE(K)$. Since \cref{lem:rsv2-gap-bound} holds with probability at least $1-\delta/2$, we have
    \begin{align*}
        \expect[\calR(K)] & \leq \bar\psi_\beta\expect[\calE(K)] \\
        & = \bar\psi_\beta\expect\Big[ \sumk\sumh \bDh(\skh,\akh;\tau_{h-1}^{k}) \Big] \\
         & = \sum_{\tau_{h-1}^{k}} \prob[\tau_{h-1}^{k}]\sumk\sumh \Dh(\skh,\akh;\tau_{h-1}^{k}) \\
        & \labelrel\leq{ineqn:stratification} \sum_{n\in [N]} \rho_n \sumk\sumh \II\{\Dh(\skh,\akh;\Tkh)\in I_n\} + \frac{\delta}{2|\beta|}\cdot HK(e^{|\beta| H}-1) \\
        & \leq \sum_{n\in [N]} \rho_n \sumk\sumh \II\{\Dh(\skh,\akh;\Tkh)\geq \rho_{n-1}\} + \frac{\delta}{2|\beta|}\cdot HK(e^{|\beta| H}-1) \\
        & \labelrel\lesssim{ineqn:gap-bound-lem} \sum_{n\in [N]} \frac{(e^{|\beta| H}-1)^2 H^3S^2A\log(4HSAK)^2}{2^n|\beta|^2\Delta_\mmin} + \frac{\delta}{|\beta|}\cdot HK(e^{|\beta| H}-1) \\
        & \lesssim \frac{(e^{|\beta| H}-1)^2 H^3S^2A}{|\beta|^2\Delta_\mmin} \log(2HSAK)^2,
    \end{align*}
    where step \eqref{ineqn:stratification} follows from stratifying the range of $\Dkh$ into $N \defeq \lceil \log_2(\frac{1}{|\beta|}(e^{|\beta|H}-1)/\Delta_\mmin)\rceil$ slices with end points $\{\rho_n\}_{n=1}^N$, where we define $\rho_n \defeq 2^n\Delta_\mmin$ and interval $I_n \defeq [\rho_{n-1},\rho_n)$ for all $n\in[N]$; step \eqref{ineqn:gap-bound-lem} follows from \cref{lem:count-bound}; the last inequality is due to $\Delta_\mmin \leq \frac{1}{|\beta|} (e^{|\beta| H}-1)$ and taking $\delta = \frac{1}{HK}$.

\subsubsection{Proof of \cref{lem:rsv2-gap-bound} \label{sec:proof-rsv2-gap-bound}}

For any $h\in[H]$ and $k\in[K]$, we have $\Vsh(\skh) = \Qsh(\skh,\pi_h^*(\skh))$ and $\Dh(\skh,\akh;\Tkh) \leq \frac{1}{|\beta|}(e^{|\beta|H}-1)$. Let us define $N \defeq \lceil \log_2(\frac{1}{|\beta|}(e^{|\beta|H}-1)/\Delta_\mmin)\rceil$. Following \cref{lem:clip}, with probability at least $1-\delta/2$, we have
\begin{align*}
    \sumk\sumh \Dh(\skh,\akh;\Tkh) & \labelrel\leq{ineqn:peeling} \sumk\sumh\sum_{n\in[N]} \rho_n\cdot \II\{\Dh(\skh,\akh;\Tkh)\in I_n\} \\
    & \labelrel\leq{ineqn:q-func-def} \sum_{n\in[N]} \rho_n \sumk\sumh \II\{\Dpikh(\skh,\akh;\Tkh) \geq \rho_{n-1}\} \\
    & \labelrel\lesssim{ineqn:clip} \sum_{n\in[N]} \rho_n \frac{(e^{|\beta| H}-1)^2 H^3S^2A\log(2HSAK/\delta)^2}{4^{n-1}|\beta|^2\Delta_\mmin^2} \\
    & \lesssim \frac{(e^{|\beta| H}-1)^2 H^3S^2A\log(2HSAK/\delta)^2}{|\beta|^2\Delta_\mmin},
\end{align*}
where step \eqref{ineqn:peeling} is due to the peeling argument that stratifies the range of $\Delta_\mmin$ into $N$ slices with end points $\{\rho_n\}_{n=1}^N$, step \eqref{ineqn:q-func-def} follows from $\Qsh(\skh,\akh) \geq \Qpih(\skh,\akh)$, and step \eqref{ineqn:clip} follows from \cref{lem:clip}.

\begin{lemma}\label{lem:clip}
    Under \cref{alg:rsv2}, with probability at least $1-\delta/2$, we have for any $n\in\ZZ_+$
    \begin{align*}
        \sumk\sumh \II\{\Dpikh(\skh,\akh;\Tkh) \geq \rho_n\} \lesssim \frac{(e^{|\beta| H}-1)^2 H^3S^2A}{4^n|\beta|^2\Delta_\mmin^2} \log(2HSAK/\delta)^2.
    \end{align*}
\end{lemma}

\begin{proof}
    Let us denote $M_{h,n}$ to be the number of episodes such that the sub-optimality of the episode at step $h$ is no less than $\rho_n$, \ie, $M_{h,n} \defeq \sumk \II\{\Dpikh(\skh,\akh;\Tkh) \geq \rho_n\}$. Especially, $k_1<\ldots<k_{M_{h,n}}<k$ denote the selected indices of previous episodes such that $\Delta_h^{\pi^{k_i}}(s_h^{k_i},a_h^{k_i};\tau_{h-1}^{k_i}) \geq \rho_n$ at step $h$, and we further define $R_h^{k_i} \defeq \sum_{j=1}^h r_j(s_j^{k_i},a_j^{k_i})$ to be the sum of rewards for the first $h$ steps within the selected episodes. For the convenience of notation, we use $\vartheta$ to denote the logarithmic factor $\log(2HSAK/\delta)$.
    Let us also define a shorthand $[\calP_h V](s,a) \defeq \expect_{s'}[V(s')]$ with respect to $\calP_h$ for any value function $V:\calS\to\RR$ and state-action pair $(s,a)\in\calS\times\calA$. 
    
    Notice that we can make a recursive upper bound on the gap between the optimistic value function $Q_h^{k_i}$ and policy controlled value function $Q_h^{\pi^{k_i}}$ as follows:
    \begin{align*}
        \sum_{i\in[M_{h,n}]} \bar\psi_\beta e^{\beta\cdot R_{h-1}^{k_i}} (e^{\beta\cdot Q_h^{k_i}} - e^{\beta\cdot Q_h^{\pi^{k_i}}})(s_h^{k_i},a_h^{k_i}) & \labelrel\leq{ineqn:add-substract} \bar\psi_\beta\sum_{i\in[M_{h,n}]}e^{\beta\cdot R_{h-1}^{k_i}} e^{\beta\cdot r_h(s_h^{k_i},a_h^{k_i})}(e^{\beta\cdot Q_{h+1}^{k_i}} - e^{\beta\cdot Q_{h+1}^{\pi^{k_i}}})(s_{h+1}^{k_i},a_{h+1}^{k_i}) \\
        & \qquad + \bar\psi_\beta\sum_{i\in[M_{h,n}]} e^{\beta\cdot R_{h-1}^{k_i}} 2b_h^{k_i} + \bar\psi_\beta\sum_{i\in[M_{h,n}]} e^{\beta\cdot R_{h-1}^{k_i}} \zeta_{h+1}^{k_i} \\
        & \leq \sum_{i\in[M_{h,n}]}\bar\psi_\beta e^{\beta\cdot R_h^{k_i}} (e^{\beta\cdot Q_{h+1}^{k_i}} - e^{\beta\cdot Q_{h+1}^{\pi^{k_i}}})(s_{h+1}^{k_i},a_{h+1}^{k_i}) \\
        & \qquad + e^{\beta\cdot (h-1)} \sum_{i\in[M_{h,n}]} 2\bar\psi_\beta b_h^{k_i} + e^{\beta\cdot (h-1)} \sum_{i\in[M_{h,n}]} \bar\psi_\beta\zeta_{h+1}^{k_i},
    \end{align*}
    where step \eqref{ineqn:add-substract} follows from the definition of \cref{alg:rsv2} with adding and subtracting $\expect_{s'} e^{\beta[r_h(\skh,\akh)+\Vkhh(s')]}$ at the same time.
    Recall that for each episode $k$ and step $h$ we defined the bonus term $b_h^k \defeq c|e^{\beta(H-h+1)}-1|\sqrt{\frac{S\vartheta}{\max\{1,N_h^k(\skh,\akh)\}}}$,
    where $c$ is a universal constant, and here we also define
    \begin{align*}
        \zeta_{h+1}^k & \defeq [\calP_h(e^{\beta[r_h(\skh,\akh)+\Vkhh(s')]} - e^{\beta[r_h(\skh,\akh)+V_{h+1}^{\pi^k}(s')]})](\skh,\akh) \\
        & \qquad - e^{\beta\cdot r_h(\skh,\akh)}(e^{\beta\cdot V_{h+1}^k} - e^{\beta\cdot V_{h+1}^{\pi^k}})(s_{h+1}^k).
    \end{align*}
    To simplify the notation, we denote $n_h^k \defeq \max\{1,N_h^k(\skh,\akh)\}$.
    Expanding the recursive inequality, with probability at least $1-\delta/2$ we have
    \begin{align}
        \sum_{i\in[M_{h,n}]} e^{\beta\cdot R_{h-1}^{k_i}} (e^{\beta\cdot Q_h^{k_i}} - e^{\beta\cdot Q_h^{\pi^{k_i}}})(s_h^{k_i},a_h^{k_i}) & \leq 2\sumh \sum_{i\in[M_{h,n}]} e^{\beta\cdot (h-1)} \bar\psi_\beta b_h^{k_i} + \sumh \sum_{i\in[M_{h,n}]} e^{\beta\cdot (h-1)} \bar\psi_\beta\zeta_{h+1}^{k_i} \nonumber \\
        & \labelrel\leq{ineqn:combine-upper-bounds} 2c(e^{|\beta| H}-1)\sqrt{2H^2S^2AM_{h,n}\vartheta^2} + (e^{|\beta| H}-1)\sqrt{2HM_{h,n}\vartheta}, \label{ineqn:opt-gap-upper}
    \end{align}
    where step \eqref{ineqn:combine-upper-bounds} follows from two upper bounds on $\sumh \sum_{i\in[M_{h,n}]}e^{\beta\cdot (h-1)} b_h^{k_i}$ and $\sumh \sum_{i\in[M_{h,n}]}e^{\beta\cdot (h-1)} \zeta_{h+1}^{k_i}$. 
    More specifically, for the summation on $b_h^{k_i}$ we have
    \begin{align*}
        \sumh \sum_{i\in[M_{h,n}]} e^{\beta\cdot (h-1)} b_h^{k_i} & \leq \sumh \sum_{i\in[M_{h,n}]} c|e^{\beta H}-1|\sqrt{\frac{S\vartheta}{n_h^{k_i}}} \\
        & \labelrel\leq{ineqn:titu} c|e^{\beta H}-1|\sqrt{S\vartheta} \sumh \sqrt{M_{h,n}}\sqrt{\sum_{i\in[M_{h,n}]}\frac{1}{n_h^{k_i}}} \\
        & \leq c|e^{\beta H}-1|\sqrt{S\vartheta} \sumh \sqrt{M_{h,n}}\sqrt{\sum_{s,a}\sum_{j=1}^{N_h^{M_{h,n}}(s,a)}\frac{1}{\max\{1,j\}}} \\
        & \labelrel\leq{ineqn:pigeon} c|e^{\beta H}-1|\sqrt{S\vartheta}\sqrt{2H^2SAM_{h,n}},
    \end{align*}
    where step \eqref{ineqn:titu} follows from the Cauchy–Schwarz inequality, and step \eqref{ineqn:pigeon} follows from the pigeonhole principle. Since each term of $e^{\beta (h-1)}\zeta_h^k$ can be controlled by $|e^{\beta (h-1)}\zeta_h^k|\leq |e^{\beta H}-1|$ for all $k\in[K]$ and $h\in[H]$, the Azuma-Hoeffding inequality gives
    \begin{align*}
        \prob\Big[ \sumh \sum_{i\in[M_{h,n}]}e^{\beta\cdot (h-1)} \zeta_{h+1}^{k_i} \geq \eps \Big] \leq \exp\Big( -\frac{\eps^2}{2HM_{h,n}(e^{\beta H}-1)^2} \Big)
    \end{align*}
    for any $\eps>0$, which means with probability at least $1-\delta/2$, 
    \begin{align*}
        \sumh \sum_{i\in[M_{h,n}]}e^{\beta\cdot (h-1)} \zeta_{h+1}^{k_i} \leq |e^{\beta H}-1|\sqrt{2HM_{h,n}\vartheta}.
    \end{align*}
    
    At the same time, we provide for the optimism gap a lower bound as follows:
    \begin{align}
        \sum_{i\in[M_{h,n}]} \bar\psi_\beta e^{\beta\cdot R_{h-1}^{k_i}} (e^{\beta\cdot Q_h^{k_i}} - e^{\beta\cdot Q_h^{\pi^{k_i}}})(s_h^{k_i},a_h^{k_i}) & \labelrel\geq{ineqn:rsv2-def} \bar\psi_\beta\sum_{i\in[M_{h,n}]} e^{\beta\cdot R_{h-1}^{k_i}} (e^{\beta\cdot Q_h^{k_i}(s_h^{k_i},\pi_h^*(s_h^{k_i}))} - e^{\beta\cdot Q_h^{\pi^{k_i}}(s_h^{k_i},a_h^{k_i})}) \nonumber \\
        & \labelrel\geq{ineqn:ucb-lem} \bar\psi_\beta\sum_{i\in[M_{h,n}]} e^{\beta\cdot R_{h-1}^{k_i}} (e^{\beta\cdot \Qsh(s_h^{k_i},\pi_h^*(s_h^{k_i}))} - e^{\beta\cdot Q_h^{\pi^{k_i}}(s_h^{k_i},a_h^{k_i})}) \nonumber \\
        & = \sum_{i\in[M_{h,n}]}|\beta|\Delta_h^{\pi^{k_i}}(s_h^{k_i},a_h^{k_i};\tau_h^{k_i}) \nonumber \\ 
        & \geq \rho_n|\beta| M_{h,n}, \label{ineqn:opt-gap-lower}
    \end{align}
    where step \eqref{ineqn:rsv2-def} is due to the construction of \cref{alg:rsv2} and step \eqref{ineqn:ucb-lem} follows from \cref{lem:ubc}.
    
    Finally, we combine the upper bound \eqref{ineqn:opt-gap-upper} and lower bound \eqref{ineqn:opt-gap-lower} of $\sum_{i\in[M_{h,n}]}\bar\psi_\beta e^{\beta\cdot R_{h-1}} (e^{\beta\cdot Q_h^{k_i}} - e^{\beta\cdot Q_h^{\pi^{k_i}}})(s_h^{k_i},a_h^{k_i})$ to get
    \begin{align*}
        \rho_n|\beta| M_{h,n} \leq 2c(e^{|\beta| H}-1)\sqrt{2H^2S^2AM_{h,n}\vartheta^2} + (e^{|\beta| H}-1)\sqrt{2HM_{h,n}\vartheta}.
    \end{align*}
    Solving for $M_{h,n}$ we get
    \begin{align*}
        M_{h,n} = \sumk \II\{\Dpikh(\skh,\akh;\Tkh) \geq \rho_n\} \lesssim \frac{(e^{|\beta| H}-1)^2 H^2S^2A\log(2HSAK/\delta)^2}{4^n|\beta|^2\Delta_\mmin^2}.
    \end{align*}
\end{proof}

\subsection{Upper Bounds for \cref{alg:rsq2}}\label{sec:rsq2-bound}

\subsubsection{Proof of \cref{thm:rsq2-bound-combined}}



\paragraph{High-probability regret bound.}
    By \cref{lem:high-prob-bound,lem:rsq2-gap-bound}, it holds with probability at least $1-\delta$ that
    \begin{align*}
        \calR(K) & \lesssim \sumk\sumh \Dh(\skh,\akh;\Tkh) + \frac{(e^{|\beta| H}-1)H\log(\log K/\delta)}{|\beta|} \\
        & \lesssim \frac{(e^{|\beta| H}-1)^2 H^4 SA\log(2HSAK/\delta)}{|\beta|^2\Delta_\mmin},
    \end{align*}
    where the last inequality is due to $\Delta_\mmin \leq \frac{1}{|\beta|}(e^{|\beta| H}-1)$.

\paragraph{Expected regret bound.}
    From \cref{lem:regret-bound} we have that $\calR(K) \leq \bar\psi_\beta \cdot \calE(K)$ and \cref{lem:rsq2-gap-bound} holds with probability at least $1-\delta/2$. The expected regret can be bounded through
    \begin{align*}
        \expect[\calR(K)] & \leq \bar\psi_\beta\expect[\calE(K)] \\
        & = \bar\psi_\beta\expect\Bigg[ \sumk\sumh \bDh(\skh,\akh;\tau_{h-1}^{\pi^k}) \Bigg] \\
        & = \sum_{\tau_{h-1}^{\pi^k}} \prob[\tau_{h-1}^{\pi^k}]\sumk\sumh \Dh(\skh,\akh;\tau_{h-1}^{\pi^k}) \\
        & \labelrel\leq{ineqn:strat} \sum_{n\in [N]} \rho_n \sumk\sumh\II\{\Dkh(\skh,\akh;\tau_{h-1}^{\pi^k}) \in I_n\} + \frac{\delta}{2|\beta|} HK(e^{|\beta| H}-1) \\
        & \labelrel\lesssim{ineqn:count-bound-lem} \sum_{n\in [N]} \frac{(e^{|\beta| H}-1)^2 H^4 SA\log(4HSAK)}{2^n|\beta|^2\Delta_\mmin} + \frac{\delta}{|\beta|} HK(e^{|\beta| H}-1) \\
        & \lesssim \frac{(e^{|\beta| H}-1)^2 H^4 SA}{|\beta|^2\Delta_\mmin}\log(HSAK), 
    \end{align*}
    where step \eqref{ineqn:strat} follows from stratifying the range of $\Dkh$ into $N \defeq \lceil \log_2(\frac{1}{|\beta|}(e^{|\beta|H}-1)/\Delta_\mmin)\rceil$ slices with end points $\{\rho_n\}_{n=1}^N$; step \eqref{ineqn:count-bound-lem} follows from \cref{lem:count-bound}; the last inequality is due to $\Delta_\mmin \leq \frac{1}{|\beta|} (e^{|\beta| H}-1)$ and taking $\delta=\frac{1}{HK}$.

\subsubsection{Proof of \cref{lem:rsq2-gap-bound} \label{sec:proof-rsq2-gap-bound}}

With the help of \cref{lem:count-bound}, it holds with probability at least $1-\delta/2$ that
\begin{align*}
    \sumk\sumh \II\{\Dkh(\skh,\akh;\tau_{h-1}^{\pi^k}) \in I_n\} \lesssim \frac{(e^{|\beta| H}-1)^2 H^4 SA\log(2HSAK/\delta)}{4^{n-1}|\beta|^2\Delta_\mmin^2}.
\end{align*}
Hence, with probability at least $1-\delta/2$, the sum of the cascaded gaps is bounded by
\begin{align*}
    \sumk\sumh \Dh(\skh,\akh;\Tkh) & \leq \sum_{n\in [N]} \rho_n  \sumk\sumh \II\{\Dkh(\skh,\akh;\tau_{h-1}^{\pi^k}) \in I_n\} \\
    & \lesssim \sum_{n\in [N]} \frac{(e^{|\beta| H}-1)^2   H^4 SA\log(2HSAK/\delta)}{2^{n-1}|\beta|^2\Delta_\mmin} \\
    & \lesssim
    \frac{(e^{|\beta| H}-1)^2 H^4 SA\log(2HSAK/\delta)}{|\beta|^2\Delta_\mmin},
\end{align*}
where $N \defeq \lceil \log_2(\frac{1}{|\beta|}(e^{|\beta|H}-1)/\Delta_\mmin)\rceil$ and the last step 
follows from an infinite sum of geometric series. Recall the definition $\rho_n \defeq 2^n\Delta_\mmin$ and interval $I_n \defeq [\rho_{n-1},\rho_n)$ for all $n\in[N]$.

\begin{lemma}\label{lem:count-bound}
    Under \cref{alg:rsq2}, with probability at least $1-\delta/2$, it holds that for any $n\in\ZZ_+$
    \begin{align*}
        \sumk\sumh \II\{\Dkh(\skh,\akh;\tau_{h-1}^{\pi^k}) \in I_n\} \lesssim \frac{(e^{|\beta| H}-1)^2 H^4SA}{4^n|\beta|^2\Delta_\mmin^2} \log(2HSAK/\delta).
    \end{align*}
\end{lemma}

\begin{proof}
    We focus on the case of $\beta>0$; the case for $\beta<0$ follows a similar argument. We denote the shorthand $\vartheta\defeq \log(2HSAK/\delta)$.
    For every $h\in[H]$ and $n\in[N]$, we define
    \begin{align*}
        \overline M_{h,n} \defeq \sumk \II\{\Dkh(\skh,\akh;\tau_{h-1}^{\pi^k}) \in I_n\}
    \end{align*}
    to be the number of episodes where the corresponding gap falls into the interval $I_n$.
    For any $i\in[\overline M_{h,n}]$, we denote $k_i$ to be the \ith episode with gap $\Dkh(\skh,\akh;\tau_{h-1}^{\pi^k})$ lying in the interval $I_n$.
    
    Recall that $ N_h^k(\skh,\akh)$ 
    is the number of visits on state-action pair $(s_h^k,a_h^k)$ at step $h$ prior to episode $k$, and $\gamma_{h,t} \defeq 2\sum_{i\in[t]} \alpha_t^i b_{h,i}$ is the corresponding bonus term for any given $t$. For the time being, we only consider step $h$, and we will ignore some of the subscripts on $h$ for simplicity of notation. In particular, we define $t_i \defeq N_h^{k_i}(s_h^{k_i},a_h^{k_i})$ and $\varkappa(s,a,j)$ to be the episode where $(s,a)$ is visited for the \jth time.
    We first apply \cref{lem:gap-bound} to get an upper bound with three components:
    \begin{align*}
        \sum_{i\in[\overline M_{h,n}]} \beta\Dh^{k_i}(s_h^{k_i},a_h^{k_i};\tau_{h-1}^{\pi^{k_i}}) & \leq e^{\beta(h-1)} \sum_{i\in[\overline M_{h,n}]} \left( e^{\beta\cdot Q_h^{k_i}(s_h^{k_i})} - e^{\beta\cdot\Qsh(s_h^{k_i},a_h^{k_i})} \right) \\
        & \leq e^{\beta(h-1)} \sum_{i\in[\overline M_{h,n}]} \alpha_{t_i}^0 (e^{\beta (H-h+1)} - 1) + 2e^{\beta(h-1)} \sum_{i\in[\overline M_{h,n}]} \gamma_{h,t_i} \\
        & \qquad + e^{\beta(h-1)} \sum_{i\in[\overline M_{h,n}]} \sum_{\ell\in[t_i]} \alpha_{t_i}^\ell  \cdot \left( e^{\beta\cdot V_{h+1}^{k_\ell}(s_{h+1}^{k_\ell})} - e^{\beta\cdot V_{h+1}^*(s_{h+1}^{k_\ell})} \right).
    \end{align*}
    Especially, the first term on the RHS can be bounded by the number of state-action pairs, \ie,
    \begin{align}\label{ineqn:empirical-gap-upper-bound-term1}
        \sum_{i\in[\overline M_{h,n}]} \alpha_{t_i}^0 (e^{\beta (H-h+1)} - 1) & \leq \sum_{i\in[\overline M_{h,n}]} (e^{\beta (H-h+1)} - 1)\cdot \II\{t_i = 0\} \leq (e^{\beta (H-h+1)} - 1)SA,
    \end{align}
    where $\alpha_{t_i}^i = 1$ only if $N_h^{k_i}(s_h^{k_i},a_h^{k_i}) = 0$, and $(s_h^{k_i},a_h^{k_i})\in\calS\times\calA$ only has $SA$ choices. The second term can be similarly controlled by 
    \begin{align*}
        2e^{\beta(h-1)}\sum_{i\in[\overline M_{h,n}]} \gamma_{h,t_i} & \leq 2e^{\beta(h-1)} \sum_{i\in[\overline M_{h,n}]} 4c (e^{\beta(H-h+1)}-1) \sqrt{\frac{H\vartheta}{t_i}} \\
        & \leq 8 e^{\beta(h-1)}(e^{\beta (H-h+1)}-1)c \sqrt{H\vartheta} \sum_{i\in[\overline M_{h,n}]} \frac{1}{\sqrt{N_h^{k_i}(s_h^{k_i},a_h^{k_i})}} \\
        & \leq 8 e^{\beta(h-1)}(e^{\beta (H-h+1)}-1)c \sqrt{H\vartheta} \sum_{(s,a)\in\calS\times\calA}\sum_{j=2}^{N_h^K(s,a)} \frac{\II\{\exists i\in[\overline M_{h,n}]\st\varkappa(s,a,j)=k_i\}}{\sqrt{j-1}},
    \end{align*}
    where for each state-action pair $(s,a)$, the weighted sum $\sum_{i=2}^{N_h^K(s,a)} \II\{\exists i\in[\overline M_{h,n}]\st\varkappa(s,a,j)=k_i\}/\sqrt{i-1}$ can be further bounded through
    \begin{align*}
        \sum_{i=2}^{N_h^K(s,a)} \frac{\II\{\exists i\in[\overline M_{h,n}]\st\varkappa(s,a,j)=k_i\}}{\sqrt{j-1}} & \leq \sum_{i=1}^{L_{s,a}} \frac{1}{\sqrt{i}} \leq 2\sqrt{L_{s,a}},
    \end{align*}
    where $L_{s,a} \defeq \sum_{j=1}^{N_h^K(s,a)} \II\{\exists i\in[\overline M_{h,n}]\st\varkappa(s,a,j)=k_i\}$. Then we can also get an upper bound on the second term as
    \begin{align}
        2e^{\beta(h-1)}\sum_{i\in[\overline M_{h,n}]} \gamma_{h,t_i} & \leq 16 e^{\beta(h-1)}(e^{\beta (H-h+1)}-1)c \sqrt{H\vartheta} \sum_{(s,a)\in\calS\times\calA}\sqrt{L_{s,a}} \nonumber \\
        & \labelrel\leq{ineqn:sum-l} 16 e^{\beta(h-1)}(e^{\beta (H-h+1)}-1)c \sqrt{SA\overline M_{h,n} H\vartheta}, \label{ineqn:empirical-gap-upper-bound-term2}
    \end{align}
    where step \eqref{ineqn:sum-l} follows from $\sum_{(s,a)\in\calS\times\calA} L_{s,a} = \overline M_{h,n}$.
    For the third term, by rearranging the order of summations and taking advantage of the fact that $\Vkh(\skh) = \Qkh(\skh,\akh)$ and $\Vsh(\skh) \geq \Qshh(\skhh,\akhh)$, we get
    \begin{align*}
        & \sum_{i\in[\overline M_{h,n}]} \sum_{\ell\in[t_i]} \alpha_{t_i}^\ell   \left( e^{\beta\cdot V_{h+1}^{k_\ell}(s_{h+1}^{k_\ell})} - e^{\beta\cdot V_{h+1}^*(s_{h+1}^{k_\ell})} \right) \\
        =\ & \sum_{\ell\in[K]} \left( e^{\beta\cdot V_{h+1}^\ell} - e^{\beta\cdot V_{h+1}^*} \right) (s_{h+1}^\ell) \sum_{j=N_h^\ell(s_h^\ell,a_h^\ell)+1}^{N_h^K(s_h^\ell,a_h^\ell)} \II\{\exists i\in[\overline M_{h,n}]\st \varkappa_h(s_h^\ell,a_h^\ell,j)=k_i\}\cdot \alpha_j^{N_h^\ell(s_h^\ell,a_h^\ell)+1} \\
        \leq\ & \sum_{\ell\in[K]} \left( e^{\beta\cdot Q_{h+1}^\ell} - e^{\beta\cdot Q_{h+1}^*} \right) (s_{h+1}^\ell) \sum_{j=N_h^\ell(s_h^\ell,a_h^\ell)+1}^{N_h^K(s_h^\ell,a_h^\ell)} \II\{\exists i\in[\overline M_{h,n}]\st \varkappa_h(s_h^\ell,a_h^\ell,j)=k_i\}\cdot \alpha_j^{N_h^\ell(s_h^\ell,a_h^\ell)+1}.
    \end{align*}
    Denote $\phi^\ell \defeq \sum_{j=N_h^\ell(s_h^\ell,a_h^\ell)+1}^{N_h^K(s_h^\ell,a_h^\ell)} \II\{\exists i\in[\overline M_{h,n}]\st \varkappa_h(s_h^\ell,a_h^\ell,j)=k_i\}\cdot \alpha_j^{N_h^\ell(s_h^\ell,a_h^\ell)+1}$, and the above inequality turns into 
    \begin{align*}
        \sum_{i\in[\overline M_{h,n}]} \sum_{\ell\in[t_i]} \alpha_{t_i}^\ell   \left( e^{\beta\cdot V_{h+1}^{k_\ell}(s_{h+1}^{k_\ell})} - e^{\beta\cdot V_{h+1}^*(s_{h+1}^{k_\ell})} \right) \leq \sum_{\ell\in[K]} \phi^\ell \left( e^{\beta\cdot Q_{h+1}^\ell} - e^{\beta\cdot Q_{h+1}^*} \right) (s_{h+1}^\ell).
    \end{align*}
    Recall that each element in $\{\phi^\ell\}_{\ell\in[K]}$ is bounded by $1+\frac{1}{H}$, and expand the recursive inequality
    \begin{align*}
        e^{\beta(h-1)} \sum_{i\in[\overline M_{h,n}]} \left( e^{\beta\cdot Q_h^{k_i}(s_h^{k_i})} - e^{\beta\cdot\Qsh(s_h^{k_i},a_h^{k_i})} \right) & \leq (e^{\beta H}-1)SA + 16 (e^{\beta H}-1)c \sqrt{SA\overline M_{h,n} H\vartheta} \\
        & \qquad + e^{\beta(h-1)}\sum_{\ell\in[K]} \phi^\ell \left( e^{\beta\cdot Q_{h+1}^\ell} - e^{\beta\cdot Q_{h+1}^*} \right) (s_{h+1}^\ell)
    \end{align*}
    to get 
    \begin{align}
        e^{\beta(h-1)} \sum_{i\in[\overline M_{h,n}]} \left( e^{\beta\cdot Q_h^{k_i}(s_h^{k_i})} - e^{\beta\cdot\Qsh(s_h^{k_i},a_h^{k_i})} \right) & \leq \sum_{w=0}^{H-h} SA(e^{\beta H}-1)\left( 1+\frac{1}{H} \right)^{w} \nonumber \\
        & \qquad + \sum_{w=0}^{H-h} 16 (e^{\beta H}-1)c (1+1/H)^{w}\sqrt{SA\overline M_{h,n} H\vartheta} \nonumber \\
        & \leq eHSA(e^{\beta H}-1) + 16eH(e^{\beta H}-1)c\sqrt{SA\overline M_{h,n}H\vartheta}. \label{ineqn:empirical-gap-upper-bound-term3}
    \end{align}
    With \eqref{ineqn:empirical-gap-upper-bound-term1}, \eqref{ineqn:empirical-gap-upper-bound-term2}, and \eqref{ineqn:empirical-gap-upper-bound-term3}, we obtain the complete upper bound
    \begin{align}
        \sum_{i\in[\overline M_{h,n}]} \beta\Dh^{k_i}(s_h^{k_i},a_h^{k_i};\tau_{h-1}^{\pi^{k_i}}) & \leq e^{\beta(h-1)} \sum_{i\in[\overline M_{h,n}]} \left( e^{\beta\cdot Q_h^{k_i}(s_h^{k_i})} - e^{\beta\cdot\Qsh(s_h^{k_i},a_h^{k_i})} \right) \nonumber \\
        & \leq (e^{\beta H}-1)(eHSA  + 16eHc\sqrt{SA\overline M_{h,n}H\vartheta}). \label{ineqn:empirical-gap-upper-bound}
    \end{align}
    On the other side, we can also obtain a lower bound on the sum of gaps following the stratification of the empirical gap $\bDkh$:
    \begin{align}\label{ineqn:empirical-gap-lower-bound}
        \sum_{i\in[\overline M_{h,n}]} \beta\Dh^{k_i}(s_h^{k_i},a_h^{k_i};\tau_{h-1}^{\pi^{k_i}}) \geq \rho_{n-1}\beta \overline M_{h,n}.
    \end{align}
    
    We combine both the upper bound \eqref{ineqn:empirical-gap-upper-bound} and the lower bound \eqref{ineqn:empirical-gap-lower-bound} on $\sum_{i\in[\overline M_{h,n}]} \beta\Dh^{k_i}(s_h^{k_i},a_h^{k_i};\tau_{h-1}^{\pi^{k_i}})$ to get $\rho_{n-1}\beta \overline M_{h,n} \leq (e^{\beta H}-1)(eHSA + 8eHc\sqrt{SA\overline M_{h,n}H\vartheta})$, which leads to a sufficient condition
    \begin{align*}
        \overline M_{h,n} = \sumk \II\{\Dkh(\skh,\akh;\tau_{h-1}^{\pi^k}) \in I_n\} \lesssim \frac{(e^{\beta H}-1)^2 H^3SA\vartheta}{4^n\beta^2\Delta_\mmin^2}.
    \end{align*}
    
    Recall that $\vartheta = \log(2HSAK/\delta)$. Sum the equation above over $h\in[H]$ to get
    \begin{align*}
        \sumk\sumh \II\{\Dkh(\skh,\akh;\tau_{h-1}^{\pi^k}) \in I_n\} \lesssim \frac{(e^{\beta H}-1)^2 H^4SA\log(2HSAK/\delta)}{4^n\beta^2\Delta_\mmin^2}.
    \end{align*}
\end{proof}

\section{Lower Bounds}\label{sec:proof-lower}

\subsection{Proof of \cref{thm:lower-combined}}

We prove the two cases of the theorem in \cref{lem:lower,lem:lower-beta}, respectively.
We construct two bandit problems such that for any policy $\pi$ the maximum regret in these two problems is lower bounded. Let us assume the first bandit machine \textsc{Bandit I} has two arms, where the first arm has reward $H-1$ with probability $p_1^{\II\{\beta>0\}}(1-p_1)^{\II\{\beta<0\}}$ and reward $0$ with probability $(1-p_1)^{\II\{\beta>0\}}p_1^{\II\{\beta<0\}}$,  whereas the second arm has reward $H-1$ with probability $p_2^{\II\{\beta>0\}}(1-p_2)^{\II\{\beta<0\}}$ and reward $0$ with probability $(1-p_2)^{\II\{\beta>0\}}p_2^{\II\{\beta<0\}}$. Similarly, the second bandit machine \textsc{Bandit II} is also assumed to have two arms with the same Bernoulli-type rewards, with corresponding probabilities $q_1$ and $q_2$, respectively.
    
It is not hard to see that a $K$-round bandit problem described above is equivalent to a $K$-episode and $H$-step \gls*{MDP} where the state space $\calS$ has three elements: initial state $s_0$, absorbing state $s_1$, and absorbing state $s_2$. At the first step, two actions $a_1,a_2\in\calA$ are available to the state $s_0$. More specifically, if one takes action $a_1$, then with probability $p_1^{\II\{\beta>0\}}(1-p_1)^{\II\{\beta<0\}}$ for \textsc{Bandit I} (or $q_1^{\II\{\beta>0\}}(1-q_1)^{\II\{\beta<0\}}$ for \textsc{Bandit II}) the environment transitions into state $s_1$ and with probability $(1-p_1)^{\II\{\beta>0\}}p_1^{\II\{\beta<0\}}$ for \textsc{Bandit I} (or $(1-q_1)^{\II\{\beta>0\}}q_1^{\II\{\beta<0\}}$ for \textsc{Bandit II}) it transitions into state $s_2$. Similarly, if one takes action $a_2$, the environment transitions according to $p_2$ for \textsc{Bandit I} (or $q_2$ for \textsc{Bandit II}). Moreover, we define reward function $r_h(s_0,a) = 0$, $r_h(s_1,a) = 1$, and $r_h(s_2,a) = 0$. In short, taking action $a_1$ is equivalent to pulling the first arm on the corresponding bandit machine and taking action $a_2$ is equivalent to pulling the second arm.
    
Now we start focusing on the lower bound analysis of the bandit problem. In particular, we define the transition probability $p_1$, $p_2$, $q_1$, and $q_2$ such that the first arm is optimal on \textsc{Bandit I} while the second arm is optimal on \textsc{Bandit II}, \ie, 
\begin{align*}
    p_2 = u_{\beta,H},\quad p_1 = q_1 = p_2+ (-1)^{\II\{\beta<0\}} \xi,\quad q_2 = p_2+(-1)^{\II\{\beta<0\}}\cdot 2\xi,
\end{align*}
where we select a positive quantity $\xi \leq \frac{1}{4}u_{\beta,H}$. The quantity $u_{\beta,H}$ is set to be $e^{-|\beta|(H-1)}$ for \cref{lem:lower} and $\frac{1}{H}$ for \cref{lem:lower-beta}.

Due to the design of the \glspl*{MDP}, $\bDh(s,a;\tau_{h-1}) = 0$ for any $h \geq 2$ and $\tau_{h-1}$. For state $s_1$, the minimal sub-optimality gap is therefore given by $\xoverline\Delta_1(s_1,a) \defeq \xoverline\Delta_1(s_1,a; \trj_0) =  \frac{1}{|\beta|}(e^{\beta\cdot V_1^*(s_1)} - e^{\beta\cdot Q_1^*(s_1,a)})$ for some action $a$, which it is by design determined by $a$ and the randomness of the environment. More specifically, $\xoverline\Delta_1(s,a) = 0$ if it takes the optimal action $a = a^*$ and the only non-zero sub-optimality gap is given by the action $a = a'$ that takes the sub-optimal arm $a'$, \ie,
\begin{align*}
    \xoverline\Delta_\mmin & = \frac{1}{|\beta|} \left| e^{\beta\cdot V_1^*(s_1)} - e^{\beta\cdot Q_1^*(s_1,a')} \right| \\
    & = \frac{1}{|\beta|} \left| p_1 e^{\beta(H-1)} + (1-p_1) - p_2 e^{\beta(H-1)} - (1-p_2) \right| \\
    & = \frac{1}{|\beta|} \left| (p_1-p_2)e^{\beta(H-1)} - (p_1-p_2) \right| \\
    & = \frac{1}{|\beta|} |e^{\beta(H-1)}-1|\xi.
\end{align*}

\cref{thm:lower-combined} follows directly by combining \cref{lem:lower,lem:lower-beta}.

\begin{lemma}\label{lem:lower}
    If $|\beta|(H-1) \geq \log 4$, $\Delta_\mmin \leq \frac{1}{8|\beta|}$, and $K \asymp \frac{1}{|\beta|^2\Delta_\mmin^2} (e^{|\beta|(H-1)}-1)$, then the regret of any policy obeys
    \begin{align*}
        \expect[\calR(K)] \gtrsim \frac{e^{|\beta|(H-1)}-1}{|\beta|^2\Delta_\mmin}.
    \end{align*}
\end{lemma}

\begin{proof}
    Applying \cref{lem:lower-bandit} with $K = \lfloor p_2(1-p_2)/\xi^2\rfloor$, we get
    \begin{align*}
        \expect[\calR(K)] & \gtrsim \frac{e^{|\beta|(H-1)}-1}{|\beta|}\cdot \frac{p_2(1-p_2)}{\xi} \\
        & \labelrel\gtrsim{ineqn:p1} \frac{e^{|\beta|(H-1)}-1}{|\beta|}\cdot \frac{p_2}{\xi} \\
        & \labelrel\gtrsim{ineqn:sub-gap} \frac{e^{|\beta|(H-1)}-1}{|\beta|}\cdot \frac{p_2|e^{\beta(H-1)}-1|}{|\beta|\xoverline\Delta_\mmin} \\
        & \labelrel\gtrsim{ineqn:gap} \frac{e^{|\beta|(H-1)}-1}{|\beta|}\cdot \frac{p_2(e^{|\beta|(H-1)}-1)}{|\beta|\Delta_\mmin} \\
        & = \frac{e^{|\beta|(H-1)}-1}{|\beta|}\cdot \frac{1-p_2}{|\beta|\Delta_\mmin} \\
        & \labelrel\gtrsim{ineqn:p2} \frac{e^{|\beta|(H-1)}-1}{|\beta|^2\Delta_\mmin},
    \end{align*}
    where step \eqref{ineqn:p1} and step \eqref{ineqn:p2} follow from $1-p_2 \geq \frac{1}{2}$, step \eqref{ineqn:sub-gap} follows from $\xoverline\Delta_\mmin = |e^{\beta(H-1)}-1|\xi/|\beta|$, and step \eqref{ineqn:gap} is due to the definition of $\Delta_\mmin$, and the equality follows from $p_2 = e^{-|\beta|(H-1)}$.
\end{proof}

\begin{lemma}\label{lem:lower-beta}
    If $H\geq 8$, $|\beta|(H-1)\leq \log H$, $\Delta_\mmin \leq \frac{1}{4|\beta|H}(e^{|\beta|(H-1)}-1)$ and the number of episodes $K \asymp \frac{1}{H|\beta|^2\Delta_\mmin^2} (e^{|\beta|(H-1)}-1)^2$, the regret of any policy obeys
    \begin{align*}
        \expect[\calR(K)] \geq \frac{H}{\Delta_\mmin}.
    \end{align*}
\end{lemma}

\begin{proof}
    Similar to the proof of \cref{lem:lower}, we have $\xoverline\Delta_\mmin = \frac{1}{|\beta|}|e^{\beta(H-1)}-1|\xi$. Note that we have $\xi \leq \frac{1}{4H}$ satisfied as $K\geq 16H$. Apply \cref{lem:lower-bandit} and take $K = \lfloor p_2(1-p_2)/\xi^2\rfloor$, then it yields
    \begin{align*}
        \expect[\calR(K)] & \gtrsim \frac{ e^{|\beta|(H-1)}-1}{|\beta|}\cdot \frac{p_2(1-p_2)}{\xi} \\
        & \labelrel\gtrsim{ineqn:p2-def} \frac{e^{|\beta|(H-1)}-1}{|\beta|}\cdot\frac{1}{H\xi} \\
        & = \frac{e^{|\beta|(H-1)}-1}{|\beta|}\cdot\frac{|e^{\beta(H-1)}-1|}{H|\beta|\xoverline\Delta_\mmin} \\
        & = \frac{(e^{|\beta|(H-1)}-1)^2}{|\beta|^2H\Delta_\mmin} \\
        & \labelrel\gtrsim{ineqn:exp-lower} \frac{(H-1)^2}{H\Delta_\mmin} \\
        & \gtrsim \frac{H}{\Delta_\mmin},
    \end{align*}
    where step \eqref{ineqn:p2-def} follows from $p_2 = \frac{1}{H}$ and step \eqref{ineqn:exp-lower} follows from $e^x-1\geq x$ for $x>0$ and $e^{|\beta|(H-1)}\leq H$.
\end{proof}

\begin{lemma}\label{lem:lower-bandit}
    For either case of 
    \begin{enumerate}
        \item $H \geq 2$,
        $|\beta|(H-1) \geq \log 4$, $p_2 = e^{-|\beta|(H-1)}$, and $0 < \xi \leq \frac{1}{4}e^{-|\beta|(H-1)}$;
        \item
        $H\geq 8$, $|\beta|(H-1)\leq \log H$, $p_2 = \frac{1}{H}$, and $0 < \xi \leq \frac{1}{4H}$,
    \end{enumerate}
    the regret of any policy obeys
    \begin{align*}
        \expect[\calR(K)] \geq \frac{K}{64|\beta|} \xi (e^{|\beta|(H-1)}-1) \exp\Big(-\frac{8K\xi^2}{p_2(1-p_2)}\Big).
    \end{align*}
\end{lemma}

\begin{proof}
    Given the definition of the \glspl*{MDP}, for $u_{\beta,H} = e^{-|\beta|(H-1)}$, we have 
    \begin{align*}
        p_2 = e^{-|\beta|(H-1)},\quad p_1 = q_1 = p_2+ (-1)^{\II\{\beta<0\}} \xi,\quad q_2 = p_2+(-1)^{\II\{\beta<0\}}\cdot 2\xi,
    \end{align*}
    where we select a positive quantity $\xi \leq \frac{1}{4}e^{-|\beta|(H-1)}$ such that all the quantities listed above are bounded below by $\frac{1}{2}$ for $|\beta|(H-1) \geq \log 4$.
    Similarly for $u_{\beta,H} = \frac{1}{H}$, we have
    \begin{align*}
        p_2 = \frac{1}{H},\quad p_1 = q_1 = p_2+ (-1)^{\II\{\beta<0\}} \xi,\quad q_2 = p_2+(-1)^{\II\{\beta<0\}}\cdot 2\xi,
    \end{align*}
    where we select a positive quantity $\xi \leq \frac{1}{4H}$ such that all the quantities are bounded below by $\frac{1}{2}$ for $H>8$.
    
    For any such \gls*{MDP} equivalent to the above bandit models and any policy $\pi$, let us define $\Gamma_a$ to be the reward from taking action $a\in\calA$. For notational convenience, we let $a^*$ denote the optimal arm and $a'$ denote the sub-optimal arm. The regret of such \gls*{MDP} in the \kth episode is given by
    \begin{align*}
        (V_1^* - V_1^{\pi^k})(s_1) & = \Big| \frac{1}{\beta}\log\expect e^{\beta\Gamma_{a^*}} - \frac{1}{\beta}\log\Big(\sum_{a\in\calA}\prob[a^k=a]\expect e^{\beta\Gamma_a}\Big) \Big| \\
        & = \frac{1}{|\beta|} \bigg| \log\frac{\sum_{a\in\calA}\prob[a^k=a]\expect e^{\beta\Gamma_a}}{\expect e^{\beta\Gamma_{a^*}}} \bigg| \\
        & \labelrel\geq{ineqn:sub-reward} \frac{1}{|\beta|} \log\Big(1 + \frac{\prob[a^k=a'] |\expect e^{\beta\Gamma_{a'}} - \expect e^{\beta\Gamma_{a^*}}|}{\expect e^{\beta\Gamma_{a^*}}}\Big) \\
        & = \frac{1}{|\beta|} \log\Big(1 + \expect[\II\{a^k=a'\}]\frac{|\expect e^{\beta\Gamma_{a'}} - \expect e^{\beta\Gamma_{a^*}}|}{\expect e^{\beta\Gamma_{a^*}}}\Big) \\
        & \labelrel\geq{ineqn:log-derivative} \frac{1}{2|\beta|} \frac{|\expect e^{\beta\Gamma_{a'}} - \expect e^{\beta\Gamma_{a^*}}|}{\expect e^{\beta\Gamma_{a^*}}} \expect[\II\{a^k=a'\}],
    \end{align*}
    where step \eqref{ineqn:sub-reward} is due to $\expect e^{\beta\Gamma_{a^*}} \geq \expect e^{\beta\Gamma_{a'}}$ for $\beta>0$, and step \eqref{ineqn:log-derivative} is due to $\log(1+x)\geq x/2$ for $x\in[0,1]$ and $\xi \leq \frac{1}{4}p_2$. In particular, we have
    \begin{align*}
        \frac{|\expect e^{\beta\Gamma_{a'}} - \expect e^{\beta\Gamma_{a^*}}|}{\expect e^{\beta\Gamma_{a^*}}} & = \frac{|(\prob[a^*]-\prob[a'])e^{\beta(H-1)}-(\prob[a^*]-\prob[a'])|}{\prob[a^*]e^{\beta(H-1)}+(1-\prob[a^*])} \\
        & = \frac{|\xi (e^{\beta(H-1)}-1)|}{\prob[a^*]e^{\beta(H-1)}+(1-\prob[a^*])} \\
        & \labelrel\geq{ineqn:bandit-def} \frac{1}{4} \xi (e^{|\beta|(H-1)}-1),
    \end{align*}
    where step \eqref{ineqn:bandit-def} follows from the definition of the bandits and the assumptions. Notice that the inequalities hold for both cases where $u_{\beta,H} = e^{-|\beta|(H-1)}$ and $u_{\beta,H} = \frac{1}{H}$. Notably, $1-\prob[a^*]$ dominates the denominator when $\beta>0$ while being on the order of $e^{\beta(H-1)}$ when $\beta<0$.
    
    Let us denote the regret on \textsc{Bandit I} with $\calR_\mathrm{I}(K)$ and that on \textsc{Bandit II} with $\calR_\mathrm{II}(K)$. Combining the two inequalities above, we have
    \begin{align*}
        \max\{\expect[\calR_\mathrm{I}(K)]+\expect[\calR_\mathrm{II}(K)]\} & \labelrel\geq{ineqn:regret-def} \frac{1}{2}\expect[\calR_\mathrm{I}(K)]+\frac{1}{2}\expect[\calR_\mathrm{II}(K)] \\
        & \labelrel\geq{ineqn:previous-lem} \frac{1}{16|\beta|} \xi (e^{|\beta|(H-1)}-1) \sumk \big(\expect_p[\II\{a^k=a'\}] + \expect_
        q[\II\{a^k=a'\}]\big) \\
        & \geq \frac{K}{64|\beta|} \xi (e^{|\beta|(H-1)}-1) \exp\Big(-\frac{8K\xi^2}{p_2(1-p_2)}\Big),
    \end{align*}
    where step \eqref{ineqn:regret-def} follows from $\calR(K) = \sumk (V_1^* - V_1^{\pi^k})(s_1)$ for each bandit, and step \eqref{ineqn:previous-lem} follows from \cref{lem:divergence}.
\end{proof}

\begin{lemma}\label{lem:divergence}
    Under the setup of \cref{lem:lower-bandit}, we have
    \begin{align*}
        \sumk \big(\expect_p[\II\{a^k=a'\}] + \expect_
        q[\II\{a^k=a'\}]\big) \geq \frac{K}{4} \exp\Big(-\frac{8K\xi^2}{p_2(1-p_2)}\Big).
    \end{align*}
\end{lemma}

\begin{proof}
    Notice that 
    \begin{align*}
        \sumk \big(\expect_p[\II\{a^k=a'\}] + \expect_
        q[\II\{a^k=a'\}]\big) & = \expect_p\Big[\sumk \II\{a^k=a'\}\Big] + \expect_
        q\Big[\sumk \II\{a^k=a'\}\Big] \\
        & \labelrel\geq{ineqn:best-arm} \frac{K}{2} \prob_p\Big[\sumk \II\{a^k=a_1\}\leq \frac{K}{2}\Big] + \frac{K}{2} \prob_q\Big[\sumk \II\{a^k=a_1\} > \frac{K}{2}\Big],
    \end{align*}
    where step \eqref{ineqn:best-arm} is due to the assumption that the optimal arm of \textsc{Bandit I} is the first arm and the optimal arm of \textsc{Bandit II} is the second arm. Following Bretagnolle-Huber inequality (\citet{lattimore2020bandit}, Theorem 14.2), we have a lower bound in the form of an exponential divergence:
    \begin{align*}
        \prob_p\Big[\sumk \II\{a^k=a_1\}\leq \frac{K}{2}\Big] + \prob_q\Big[\sumk \II\{a^k=a_1\} > \frac{K}{2}\Big] & \geq \frac{1}{2}\exp(-D_\mathrm{KL}(\prob_p\|\prob_q)),
    \end{align*}
    and the divergence between two probability measures can be upper bounded through the following argument. Let us denote $\hat p = p_2^{\II\{\beta>0\}}(1-p_2)^{\II\{\beta<0\}}$ and $\hat q = q_2^{\II\{\beta>0\}}(1-q_2)^{\II\{\beta<0\}}$, then we have
    \begin{align*}
        D_\mathrm{KL}(\prob_p\|\prob_q) & \labelrel={ineqn:lattimore} \expect_p\Big[\sumk \II\{a^k=a'\}\Big]\cdot D_\mathrm{KL}\big(\mathrm{Ber}(\hat p)\|\mathrm{Ber}(\hat q)\big) \\
        & \labelrel\leq{ineqn:kl} K\hat p\log\Big(1+\frac{\hat p-\hat q}{\hat q}\Big) + K(1-\hat p)\log\Big(1+\frac{\hat q-\hat p}{1-\hat q}\Big) \\
        & \labelrel\leq{ineqn:log} K\hat p\frac{\hat p-\hat q}{\hat q} + K(1-\hat p)\frac{\hat q-\hat p}{1-\hat q} \\
        & = \frac{(\hat q-\hat p)^2K}{\hat q(1-\hat q)} \\
        & \labelrel\leq{ineqn:xi} \frac{8\xi^2}{p_2(1-p_2)},
    \end{align*}
    where step \eqref{ineqn:lattimore} follows from \citet[Lemma 15.1]{lattimore2020bandit}, step \eqref{ineqn:kl} follows from $\expect_p[\sumk \II\{a^k=a'\}] \leq K$ and the definition of \gls*{KL} divergence, step \eqref{ineqn:log} follows from $\log(1+x) \leq x$, and step \eqref{ineqn:xi} follows from $|p_2-q_2| = 2\xi$ and $p_2\leq q_2\leq \frac{1}{2}$ for $\beta>0$ and $\frac{1}{2}p_2\leq q_2\leq p_2\leq \frac{1}{2}$ for $\beta<0$.
\end{proof}

\section{Auxiliary Lemmas}

\begin{lemma}\label{lem:regret-bound}
    If $V_1^k(s_1) \geq V_1^{\pi^k}(s_1)$ for $k\in[K]$, then the regret is upper bounded by
    \begin{align*}
        \calR(K) \leq \bar\psi_\beta \cdot\calE(K).
    \end{align*}
\end{lemma}

\begin{proof}
    Recall that $\frac{\diff}{\diff x}\log x = \frac{1}{x}$ for all $x>0$. Especially, $\frac{\diff}{\diff x}\log x \leq 1$ for all $x\geq 1$ and $\frac{\diff}{\diff x}\log x \leq e^{|\beta|H}$ for all $x\geq e^{-|\beta|H}$. The regret can be upper bounded by the corresponding exponential regret as follows: 
    \begin{align*}
        \calR(K) & = \sumk (V_1^* - V_1^{\pi^k})(s_1^k) \\
        & \labelrel\leq{ineqn:value-func-def} \sumk (V_1^k - V_1^{\pi^k})(s_1^k) \\
        & = \sumk \frac{1}{\beta} \Big[ \log\big(e^{\beta\cdot V_1^*(s_1^k)}\big) - \log\big(e^{\beta\cdot V_1^{\pi^k}(s_1^k)}\big) \Big] \\
        & \labelrel\leq{ineqn:mean-value-thm} \sumk \frac{\bar\psi_\beta}{\beta} \Big[ e^{\beta\cdot V_1^*(s_1^k)} - e^{\beta\cdot V_1^{\pi^k}(s_1^k)} \Big] \\
        & \leq \bar\psi_\beta \cdot \calE(K),
    \end{align*}
    where step \eqref{ineqn:value-func-def} follows from the assumption that $V_1^k(s) \geq V_1^\pi(s)$ and step \eqref{ineqn:mean-value-thm} follows from mean value theorem. We provide the proof here for the sake of completeness, and similar proof can be found in \citet{fei2021exponential}.
\end{proof}

\begin{lemma}
\label{lem:ubc}
    For all $k\in[K]$, $h\in[H]$, state $s\in\calS$, and $\delta>0$, the following holds with probability at least $1-\frac{\delta}{2}$:
    \begin{align*}
        \begin{cases}
            e^{\beta\cdot\Vkh(s)} \geq e^{\beta\cdot\Vpih(s)}, \qquad \beta>0, \\
            e^{\beta\cdot\Vkh(s)} \leq e^{\beta\cdot\Vpih(s)}, \qquad \beta<0.
        \end{cases}
    \end{align*}
\end{lemma}
\begin{proof}
This is \citet[Lemma 4]{fei2021exponential}.
\end{proof}

\begin{lemma}
\label{lem:gap-bound}
    For any episode $k\in[K]$, step $h\in[H]$, and state-action pair $(\skh,\akh)\in\calS\times\calA$ such that $t = N_h^k(\skh,\akh) \geq 1$, let $\gamma_{h,t} \defeq 2\sum_{i\in[t]} \alpha_t^i b_{h,i}$ and $k_1,\ldots,k_t < k$ be the episodes in which $(\skh,\akh)$ is visited at step $h$, then it holds with probability at least $1-\delta$ for any $\beta>0$ that
    \begin{align*}
        0\leq (e^{\beta\cdot\Qkh} - e^{\beta\cdot\Qsh})(\skh,\akh) \leq \alpha_t^0 \left[ e^{\beta(H-h+1)} - 1 \right] + 2\gamma_{h,t} + \sum_{i\in[t]} \alpha_t^i e^\beta \left[ e^{\beta\cdot V_{h+1}^{k_i}(s_{h+1}^{k_i})} - e^{\beta\cdot V_{h+1}^*(s_{h+1}^{k_i})} \right]
    \end{align*}
    and for any $\beta<0$ that
    \begin{align*}
        0\leq (e^{\beta\cdot\Qsh} - e^{\beta\cdot\Qkh})(\skh,\akh) \leq \alpha_t^0 \left[ 1 - e^{\beta(H-h+1)} \right] + 2\gamma_{h,t} + \sum_{i\in[t]} \alpha_t^i \left[ e^{\beta\cdot V_{h+1}^*(s_{h+1}^{k_i})} - e^{\beta\cdot V_{h+1}^{k_i}(s_{h+1}^{k_i})} \right].
    \end{align*}
\end{lemma}
\begin{proof}
This is \citet[Lemmas 3 and 8]{fei2021exponential}.
\end{proof}

\begin{lemma}
[Freedman Inequality
(\citet{cesa2006prediction}, Lemma A.7)]
\label{lem:freedman}
    Suppose $\{Z_i\}_{i=1}^n$ be a martingale difference sequence on filtration $\{\calF_i\}_{i=1}^n$ such that $Z_i$ is $\calF_{i+1}$-measurable, $\expect[Z_i\given\calF_i] = 0$, and $|Z_i| \leq B$ for some constant $B$. Define $\chi = \sum_{i=1}^n \expect[Z_i^2\given\calF_i]$, and it follows for any $u > 0$ and $v>0$ that
    \begin{align*}
        \prob\Big[ \sum_{i=1}^n Z_i \geq u,\, \chi\leq v\Big] \leq \exp\Big(\frac{-u^2}{2v+2uB/3}\Big).
    \end{align*}
\end{lemma}

\vspace{5em}

\paragraph*{Future directions and broad impact.}

Given recent advancement in the research of deep neural networks, a promising direction of further research would be to investigate how 
neural approximation and its generalization properties 
\citep{chen2020deep,chen2020more,chen2020multiple,min2021curious} would benefit risk-sensitive RL.
Understanding and designing efficient algorithms for risk-sensitive RL in other settings, such as shortest path problems \citep{min2021learning}, off-policy evaluation 
\citep{min2021variance} and offline learning \citep{chen2021infinite}, may also be of great interest.
Moreover, as risk-sensitive RL is closely related to human learning and behaviors, it would be intriguing to study how it synthesizes with relevant areas such as meta learning and bio-inspired learning 
\citep{xu2021meta,song2021convergence}.
Last but not least, exploring how risk sensitivity could be used to augment unsupervised learning algorithms \citep{fei2018exponential,fei2018hidden,fei2020achieving,ling2019landscape} would be an important future topic as well.

\end{document}